\documentclass[letterpaper]{article} 
\usepackage[preprint]{aaai2027}  
\usepackage[hyphens]{url}  
\usepackage{graphicx} 
\usepackage{natbib}  
\usepackage{caption} 
\usepackage{algorithm}
\usepackage{algorithmic}

\usepackage{newfloat}
\usepackage{listings}
\DeclareCaptionStyle{ruled}{labelfont=normalfont,labelsep=colon,strut=off} 
\floatstyle{ruled}
\newfloat{listing}{tb}{lst}{}
\floatname{listing}{Listing}
\usepackage{mathtools} 
\usepackage{booktabs} 
\usepackage{tikz} 

\usepackage{amsmath}
\usepackage{amsfonts}
\usepackage{mathtools}
\usepackage{makecell}
\usepackage{multirow}

\usepackage{amsthm}
\usepackage{accents }
\usepackage{amssymb}
\usepackage[utf8]{inputenc}
\usepackage{xparse}
\usepackage{tabularx}

\usepackage{mathtools}
\usepackage{multirow}

\newtheorem{theorem}{Theorem}

\newtheorem{lemma}{Lemma}
\newtheorem{definition}{Definition}
\newtheorem{proposition}{Proposition}
\newtheorem{remark}{Remark}
\newtheorem{example}{Example}

\newcommand{\R}{\mathbb{R}}
\newcommand{\N}{\mathbb{N}}

\newcommand{\Pb}{\mathbb{P}}
\newcommand{\Oc}{\mathcal{O}}
\newcommand{\Gc}{\mathcal{G}}
\newcommand{\Fc}{\mathcal{F}}
\newcommand{\Hc}{\mathcal{F}}

\newcommand{\union}{\bigcup}
\DeclarePairedDelimiter\abs{\lvert}{\rvert}
\DeclarePairedDelimiter\norm{\lVert}{\rVert}

\DeclarePairedDelimiter\roundBracket{\lparen}{\rparen}
\newcommand{\round}{\roundBracket*}

\DeclarePairedDelimiter{\setBracket}{\{} {\}}
\newcommand{\setB}{\setBracket*}
\newcommand{\set}{\setBracket}

\usepackage{booktabs}
\usepackage{xcolor}

\newcommand{\closure}[1]{\overline{#1}}

\NewDocumentCommand{\dist}{o}{\IfNoValueTF{#1} {\rho} {\operatorname{d}\round{#1}} }
\newcommand{\ball}{\mathcal B}
\newcommand{\ballbrack}{\round}
\NewDocumentCommand{\B}{o m m}{\IfNoValueTF {#1} {\ball_{#2} \ballbrack {#3}} {\ball^{#1}_{#2} \ballbrack {#3}}}

\NewDocumentCommand{\jac}{mmo}{
	\IfNoValueTF {#3}{
		\frac{\partial #1}{\partial #2}
	}	
	{
		\left. \frac{\partial #1}{\partial #2}  \right \rvert_{#3}
	}
}

\newcommand{\correction}[1]{#1}

\usepackage[noabbrev,capitalize]{cleveref}

\usepackage{booktabs}

\title{Rigorous Error Certification for Neural PDE Solvers:\\ From Empirical Residuals to Solution Guarantees}
\author{
    Amartya Mukherjee\equalcontrib\corresponding,
    Maxwell Fitzsimmons\equalcontrib,
    David C. Del Rey Fernández,
    Jun Liu,
}
\affiliations{
    Department of Applied Mathematics, University of Waterloo, Waterloo, ON, Canada
}

\begin{document}

\maketitle

\begin{abstract}%
Uncertainty quantification for partial differential equations is traditionally grounded in discretization theory, where solution error is controlled via mesh/grid refinement. Physics-informed neural networks fundamentally depart from this paradigm: they approximate solutions by minimizing residual losses at collocation points, introducing new sources of error arising from optimization, sampling, representation, and overfitting. As a result, the generalization error in the solution space remains an open problem.

Our main theoretical contribution establishes generalization bounds that connect residual control to solution-space error. We prove that when neural approximations lie in a compact subset of the solution space, vanishing residual error guarantees convergence to the true solution. We derive deterministic and probabilistic convergence results and provide certified generalization bounds translating residual, boundary, and initial errors into explicit solution error guarantees.
\end{abstract}




\section{Introduction}

Partial differential equations (PDEs) play a crucial role in modelling phenomena across physics \citep{karakashian1998space}, engineering \citep{zobeiry2021physics}, finance \citep{zvan1998general}, and control \citep{meng2024physics}. However, they are conceptually difficult and computationally expensive to solve. In recent years, neural network-based approaches have attracted growing attention across machine learning, applied mathematics, and physics. In particular, physics-informed neural networks (PINNs) approximate solutions by fitting functions that satisfy the governing equations at sampled locations \citep{raissi2019physics,huang2006extreme,chen2022bridging,han2018solving,sirignano2018dgm,weinan2021algorithms};see \cite{karniadakis2021physics} for a comprehensive overview of the physics-informed learning paradigm and its applications.

In many scientific and decision-making settings, solving a PDE is not sufficient: one must also quantify the uncertainty and reliability of the computed solution.
For example, classical numerical methods such as finite difference \citep{iserles2009first}, finite element \citep{reddy2005introduction}, and spectral methods \citep{shen2011spectral} provide a well-developed theory of discretization error, where the dominant source of error arises from domain discretization.
Unlike classical discretization schemes, the loss minimized by PINNs is defined through residual evaluations at finitely many collocation points. Consequently, small training residual does not automatically imply small error in the underlying function space. The central theoretical difficulty is therefore to understand when residual control translates into convergence to the true PDE solution. Without such a link, the reliability of neural PDE solvers cannot be rigorously assessed.

While recent works have begun to bound residual errors, boundary violations, or testing errors for PINNs, a fundamental question remains largely unexplored: When does a small residual loss imply that the neural network is close to the true solution in the solution space?

Our results establish, for the first time, a direct theoretical link between residual-based training objectives and error in the solution space for PINNs. Unlike prior works that bound residual, boundary, or initial condition errors in isolation \citep{zhang2018efficient,xu2020automatic,eirasefficient}, we show how such quantities affect the true solution error. This perspective reframes uncertainty quantification for neural PDE solvers: residual control alone is insufficient; structural regularity is essential.

We complement our theoretical analysis with numerical experiments on ODEs and PDEs, including elliptic, parabolic, and hyperbolic examples. Using formal verification tools, we illustrate how error certification yields verifiable generalization guarantees and empirically stable convergence.

The main contributions of this paper are as follows:
\begin{enumerate}
    \item We conduct a convergence analysis of PINNs via compactness arguments.
    \item We provide generalization bounds based on errors that can be explicitly computed without access to the true solution.
    \item We do formal verifications on several systems and show that the true solution error is bounded above by our generalization error.
\end{enumerate}

\section{Related Works}\label{sec:related}

\textbf{Neural PDE Solvers.}
There are various machine learning-based solvers for PDEs, among which PINNs \citep{raissi2019physics} and extreme learning machines (ELMs) \citep{huang2006extreme} are widely used. PINNs incorporate the PDE residuals and initial/boundary conditions directly into the loss function, while ELMs perform a least squares approximation. Both methods demonstrate remarkable empirical performance in solving a wide variety of PDEs.
Despite their success, a rigorous understanding of when residual minimization guarantees convergence in the solution space, and how residual error translates into solution error, remains largely unresolved.

\textbf{Formally Verified Neural Networks.}
Motivated by adversarial attacks in neural networks, verifiability and perturbation analysis has been of interest to the trustworthy machine learning community. 
Satisfiability modulo theories (SMT) solvers such as dReal \citep{gao2013dreal} have been used to verify inequalities in neural networks by finding counterexamples.
Similarly, $\alpha$-CROWN \citep{zhang2018efficient} and $\beta$-CROWN \citep{wang2021beta} compute bounds in neural networks under perturbations via linear or quadratic relaxations of activation functions.
AutoLiRPA \citep{xu2020automatic} builds upon CROWN to bound neural networks provided a bounded range of perturbations.

\textbf{Formal Verifications in PDEs.}
More recently, verifications in neural networks have been extended to PDE residuals.
An early example is neural Lyapunov functions \citep{chang2019neural,zhou2022neural,liu2023towards,liu2025physics}, where SMT solvers were used to verify that the learned neural network satisfies the inequalities required for it to be a Lyapunov function. Among these, \cite{liu2023towards,liu2025physics} specifically used PINNs to learn neural Lyapunov functions. The work of \cite{eirasefficient} proposes $\delta$-CROWN, which bounds derivatives of neural networks with respect to their inputs, and therefore, bounds residual errors in PINNs. 
Finally, generalization bounds in PDEs were explored by \cite{mishra2023estimates} in various settings under some stability and training process assumptions.

Recent works have developed residual-based error bounds for neural differential equation solvers. The work of \cite{pmlr-v216-liu23b} derives computable residual-to-error bounds for PINNs across ODEs and first-order PDEs, while \cite{kongerror} proposes recursive error-network constructions for bounding worst-case approximation errors in Fokker–Planck equations. Additionally, \cite{pmlr-v286-flores25a} incorporates deterministic residual-based bounds into Bayesian PINNs to improve uncertainty calibration, and \cite{pmlr-v244-xu24a} develops operator-learning frameworks for parabolic PDEs.

Our paper builds upon these works by first establishing convergence guarantees under compactness assumptions, showing that vanishing residual error implies convergence to the true solution. We then provide generalization bounds that can be computed explicitly without knowing the true solution. 
\correction{In contrast to \cite{mishra2023estimates}, we derive bounds that are independent of the unknown true solution and are therefore computable in practice.}
Using verification tools, we bound all sources of error (residual, initial, and boundary) and combine them with theoretically derived error estimates to obtain bounds on the true solution error. Numerical examples demonstrate that these bounds are computable for a variety of ODE and PDE problems and that the true solution error is indeed bounded by our generalization bounds.

\section{Problem formulation}

Let $X\subseteq\R^{d_X}$ be compact and $Y\subseteq\R^{d_Y}$. Let $\Gc$ be a Banach space of functions from $X$ to $Y$, equipped with norm $\|\cdot\|_\Gc$, representing the solution space. Typical examples include $C(X,Y)$ with the uniform norm or $L^p(X,Y)$. 
Let $\Hc\subseteq\Gc$ denote a hypothesis class, e.g., functions representable by neural networks of a prescribed architecture type.

\subsection{Equation Operator Formulation}

Let $O:\Gc\times X\to\R^{d_O}$ be a (possibly nonlinear) differential operator defining the equation of interest. For each $f\in\Gc$, we define the residual function $O(f,\cdot)=:O_f:X\to\R^{d_O}$. Let $\Oc$ denote a complete normed function space equipped with norm $\|\cdot\|_\Oc$. We will make the following assumptions:
\begin{enumerate}
    \item $O$ admits a {unique solution} $g\in\Gc$ that satisfies $\|O(g,\cdot)\|_\Oc=0$. 
    \item The map $f\to\|O(f,\cdot)\|_\Oc$ is continuous.
\end{enumerate}
For each $f\in \Gc$, we define the \emph{solution error} as $E_g(f)=\|f-g\|_\Gc$, where $g$ is the unique solution to $O$, and the \emph{residual error} (or \emph{equation error}) as $E_O(f)=\|O(f,\cdot)\|_\Oc$. We assume that the equation $O$ and the norm $\|\cdot\|_\Oc$ is known, but the unique solution $g$ is unknown. 

Unlike classical numerical methods for PDEs, where discretizations are posed in variational or projection forms that provide stability estimates linking residual and solution error, neural PDE solvers minimize a pointwise residual evaluated at finitely many collocation points. Without such stability inequalities, small residual error does not automatically imply small solution error. This motivates the central question of this work: when does small residual error imply small solution error, and how do we rigorously quantify solution errors from residual errors?

To illustrate the abstract formulation above, we provide an example based on ordinary differential equations (ODEs) below. 

\begin{example}[Initial Value Problem]\label{ex:problem-IVP}
    Consider an initial value problem (IVP)
    \[
    \dot \phi(t) = F\circ \phi(t), \quad \phi(0)=\phi_0,
    \]
    on $[0,T]$.
    We pick $O$ to be
    \[
    O(\phi,t)= 
    \begin{pmatrix}
        \dot \phi(t) -F\circ \phi(t) \\
        \phi(0)-\phi_0
    \end{pmatrix}
    \]
    where, if $\phi: [0,T] \to \R^n$ is continuously differentiable, then $O: C^1([0,T],\R^d) \times [0,T] \to \R^{2d}$. We endow $\mathcal O$ with the uniform norm on $[0,T]$ and endow $\mathcal G$ with the norm $\norm{g}_{\mathcal G}=\max\setB{\norm{g}_{\infty},\norm{\dot g}_{\infty}}$, where the $\norm{\cdot}_{\infty}$ is the uniform norm on $[0,T]$. When the function $F:\R^{d}\to \R^{d}$ is Lipschitz, $O$ has a unique solution, which is also the unique solution to the IVP. 
\end{example}

\subsection{Training Error}

The \emph{training eror} observed at $n$ points $E_{T,n}:\Gc\to[0,\infty)$ is a quantity directly minimized by some algorithm. For instance, in Example \ref{ex:problem-IVP}, a natural choice for the training error is,
\[
E_{T,n}(\phi) = \max_{k=1,\dots,n}\setB{\norm{O(\phi,t_k)}_{\R^{2d}} }
\]
where the points $\setB{t_k:k\in \N}$ are dense in $[0,T]$ and $\norm{\cdot}_{\R^{2d}}$ is some fixed norm on $\R^{2d}$. If instead we endowed $\mathcal O$ with the $L^2$ norm, it would be more natural to use the typical ``mean squared error (MSE)" as the training error.

The hope is that, as $n\to\infty$, a small training error $E_{T,n}(f)$ implies a small residual error $E_O(f)$, which implies a small solution error $E_g(f)$. We formalize this minimal consistency requirement.

\begin{definition}\label{def:sound appox}\em 
Suppose that $O$ has a unique solution $g$ in $\Gc$. The hypothesis class $\Hc\subseteq \Gc$ is fixed, and $E_{T,n}$ is a training error.
We say that $(\Hc, E_{T,n})$ \emph{soundly approximates} the solution to $O$, if, for every sequence $\setB{f_n \in \Hc}_{n\in \N}$ satisfying $\lim_{n\to \infty}E_{T,n}(f_n)=0$, we also have $\lim_{n\to \infty}f_n =g$ in $\Gc$.
\end{definition}

The sequence $\{f_n\in\Hc\}$ represents some computed (potential) minimizers of $E_{T,n}$ for all $n\in\N$. In other words, vanishing training error implies convergence to the true solution. This property expresses a minimal reliability requirement for neural PDE solvers.

However, sound approximation does not hold in general.
We provide a counterexample below.
\begin{example}[Inadequate convergence of potential approximates]\label{ex:not sound}
    Let $X=[0,1]$ and $\mathcal G=C(X,X)$. Define $O(f,x) =f(x)$. Then, the unique solution to $O$ is the zero function $g\equiv 0$.
    For each $n\in \N$, define the points $x_{k,n}=\frac{k}{n}$ for $k=0,1,\dots,n$.  Then define 
    \[
    f_n(x)= \begin{cases}
        2^n(x-x_{k,n}),\qquad x \in [x_{k,n},x_{k,n}+2^{-n}], \\
        1,\qquad\qquad x \in [x_{k,n}+2^{-n},x_{k+1,n}-2^{-n}], \\
        -2^n(x-x_{k+1,n}),~ x \in [x_{k+1,n}-2^{-n},x_{k+1,n}], \\
    \end{cases}
    \]
    for $k=0,\dots, n$.
    Then we have that,
    \begin{align*}
        \norm{f_n}_{\infty}&= \sup_{x\in [0,1]}\abs{f_n(x)-g(x)}= f_n(0\correction{+}2^{-n})= 1, \\
        \norm{f_n}_{p}^p &= \int_{0}^{1} \abs{f_n(x)}^p dx \geq 1- n2^{-n+1}.
    \end{align*}
    In all cases, as $n\to \infty$, we have $ E_{T,n}(f_n)=0$, but
    \begin{align*}
        \norm{f_n}_{\infty} \to 1 \neq 0, && 
        \norm{f_n}_{p}^p \to 1 \neq 0,
    \end{align*}
    which implies $f_n\not \to 0$ in the norms under consideration. 
    Thus, vanishing training error does not imply convergence.
\end{example}
\begin{figure}
    \centering
    \includegraphics[width=\linewidth]{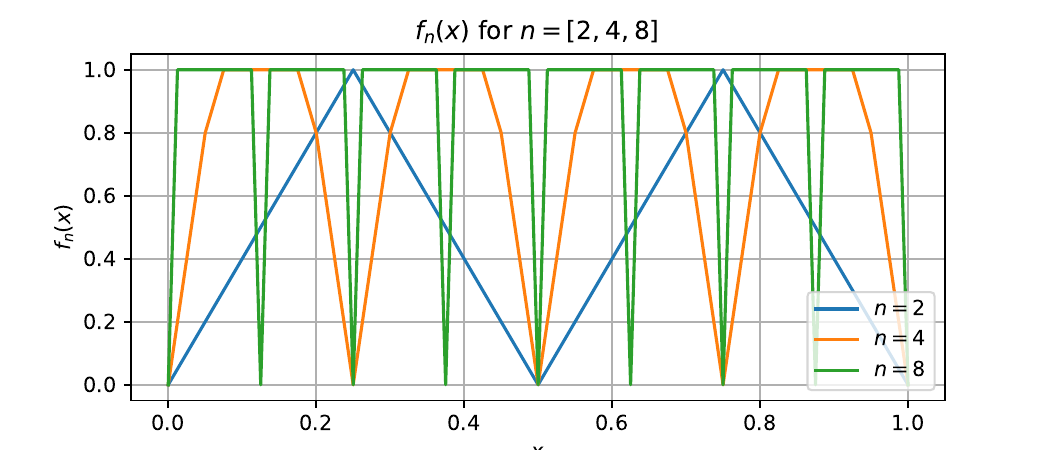}
    \caption{Illustration of $f_n(x)$ in Example \ref{ex:not sound}.}
\end{figure}
This scenario can also happen for the IVP in Example \ref{ex:problem-IVP}: we can pick $\phi_n$ as solutions to the IVP $\dot \phi_n(t) = F\circ\phi_n(t) +f_n(t)$, $\phi_n(0)=\phi_0$. 

This example demonstrates a neural overfitting phenomenon: the approximations interpolate the collocation points while oscillating arbitrarily between them. We are therefore faced with three fundamental questions:
\begin{itemize}
    \item \textbf{Consistency.} Under what structural assumptions does vanishing residual loss imply convergence to the true solution?
    \item \textbf{Error transfer.} How can residual error $E_O(f)$ be converted into bounds on the solution error $E_g(f)$ without access to the ground truth?
    \item \textbf{Certification.} Can we rigorously certify domain-wide bounds on the residual $E_O(f)$ using formal verification tools?
\end{itemize}


\section{Convergence via Compactness}\label{sec:training_error}

Example \ref{ex:not sound} demonstrates that residual minimization alone does not ensure convergence in the solution space: the approximating functions may exhibit increasingly steep oscillations while matching all collocation points. 
This suggests that residual control must be complemented by a structural condition on the hypothesis class $\Hc$. In this section, we show that this additional structural condition is compactness in $\mathcal G$. 
But first we should discuss why, in a more general problem, that compactness in $\mathcal G$ is useful. 
For a normed space $(Z,\norm{\cdot}_{Z})$ and $S\subseteq Z$, we denote $\closure{S}$ as the closure of $S$ in $Z$. We refer the reader to Appendix \ref{app:convergence} for the proofs of all the theoretical results here.

\begin{proposition}\label{thm:no training compact}
    Let $O:\mathcal G \times X \to \R^{d_O}$ be an equation, with $E_{\mathcal O}(f)=\norm{O(f,\cdot)}_{\mathcal O}$ continuous. 
    If $O$ has a unique solution $g\in \mathcal G$ then,
    \begin{align*}
        \closure{\setB{f_n:n\in \N}} \text{ is compact}&\\
        \text{and }\lim_{n\to \infty} E_{\mathcal O}(f_n) =0 &\iff \lim_{n\to \infty} f_n= g \text{ in } \mathcal G.
    \end{align*}
\end{proposition}	  

Assume that the training error $E_{T,n}$ approximates the true equation error $E_{\mathcal O}$. Proposition \ref{thm:no training compact} considers an ideal case of the sound approximation property in Definition \ref{def:sound appox}, where we do not need to approximate the true equation error and instead have direct access to it. In this case, to ensure sound approximation, it is both necessary and sufficient that the sequence of approximations lies in a compact subset of $\Gc$. Therefore, in our less ideal case, where the training error $E_{T,n}$ is used as a proxy for the true equation error $E_{\mathcal O}$, we should also require compactness. 

We now investigate the sound approximation property under different norms.



The case where $\mathcal G$ has the uniform norm leads to nice convergence theorems with small modifications to the training error. We provide a background on compactness in the uniform norm in Appendix \ref{sec:unifrom norm}. 
\correction{In summary, the compactness of a set $\mathcal F$ can be sufficiently achieved on connected domains by enforcing every $f\in \mathcal F$ to have a uniformly bounded Lipschitz constant and for every function to have uniformly bounded interpolation error of a fixed data point $(x_0,y_0)$. }

\begin{theorem}[Uniform-uniform convergence]\label{thm:sound equtaion solve}
    Let $\mathcal G= C(X,Y)$ with the uniform norm, where $X$ is connected, compact, and $\setB{a_n: n\in \N}$ is a sequence dense on $X$ (i.e., $\closure{\setB{a_n: n\in \N}}=X$).
    Suppose that $O:C(X,Y) \times X \to \R^d$ is a continuous function, we endow the space $\mathcal O =\setB{O(f,\cdot):f\in \mathcal G}$ with the uniform norm and let $\norm{\cdot}_{\R^d}$ be some fixed norm on $\R^d$.
    
    Assume that $O$ has a unique solution. Further, consider a hypothesis class $\Hc\subseteq\mathcal G$ and let $L: \mathcal F \to [0,\infty)$ be a function which satisfies
    \[
    \norm{f(x)-f(z)}_{Y} \leq  L(f)\norm{x-z}_X
    \]
    for all $x,z\in X$ and each $f\in \mathcal F$, where $\norm{\cdot}_X$ and $\norm{\cdot}_Y$ are some fixed norms on $X$ and $Y$, respectively. That is, $L(f)$ is a Lipschitz constant of $f\in \mathcal F$. Let $L_{\max}>0$ be some fixed number and let $x_0 \in X$, $y_0\in Y$ be some points.

    Let $\tilde {E}_{T,n}^{p}(f)=\max_{k=1,\dots, n}\setB{\norm{O(f,a_k)}_{\R^d}}$ for $p=\infty$ and $\tilde {E}_{T,n}^{p}(f)=\sum_{k=1}^n{\norm{O(f,a_k)}_{\R^d}^p}$ for $p> 0$. 

     Define the training error $E_{T,n}^{p}:\mathcal F \to \R$ for $p\in (0,\infty]$ by 
        \begin{align*}
        E_{T,n}^{p}(f)
        =&~\tilde {E}_{T,n}^{p}(f) +\max\setB{0,L(f)-L_{\max}} \\
        &+\norm{f(x_0)-y_0}_Y
        \end{align*}
        for $n\in \N$. Then $(\mathcal F,E_{T,n}^{p})$ soundly approximates the solution to $O$. \label{thm:sound equtaion solve.psum} %
    
\end{theorem}
Note that the function $L(f)$ must be obtained analytically in practice. Moreover, $L(f)$ only needs to exceed the minimal Lipschitz constant of $f$, so crude but easily computable estimates suffice. In Appendix \ref{sec:Estimation of Lipschitz Constants}, we list some preliminary results on estimation of Lipschitz constants for neural network functions. The term corresponding to the singular data point $(x_0,y_0)$ may be unnecessary depending on the definition of $O$ (see Example \ref{ex:problem-IVP}, where such information is already encoded in the definition of $O$).


The additional terms, $\max\setB{0,L(f)-L_{\max}}+\norm{f(x_0)-y_0}_Y$, exist solely to make the sequence $\setB{f_n}_{n\in \N}$ lie in a compact subset of $\mathcal G$. Therefore, we dub this a ``uniform compactness term''. 


One can imagine that when $\mathcal G$ is endowed with a norm other than the uniform norm, a different term could be used to enforce compactness. Note that even when $\mathcal G$ has the uniform norm, there are other possible ``uniform compactness terms'' that ensure the sound approximation property in Definition \ref{def:sound appox}. The compactness term used in Theorem \ref{thm:sound equtaion solve} was chosen because it is conceivably usable in practice. We discuss these compactness terms in more detail in Appendix  \ref{sec:compactness_terms}. 


It is common to use the ``mean squared error (MSE)'' in training practice. Theorem \ref{thm:sound equtaion solve} does not address this case. If we add additional assumptions, we can get a similar convergence result for the MSE as follows.   

\begin{theorem}[Uniform-$L^2$ convergence]\label{thm:MSE sound equtaion solve}
    Let $\mathcal G= C(X,Y)$ be equipped with the uniform norm, where $X$ is connected, compact,  and $\setB{a_n: n\in \N}$ is a sequence dense on $X$ with the additional requirement that, for all $n\in \N$, 
    \[
    \delta_n= \inf\setB{\delta >0: X\subseteq \union_{k=1}^n\B[X]{\delta}{a_k}},
    \]
    and
    \[
    C=\sup\setB{ n\mu(\B[X]{\delta_n}{a_1}) : n\in \N} <\infty,
    \]
    where $\mu$ is the Lebesgue measure and $\B[X]{\delta}{x}$ is the open ball of radius $\delta >0$ centered at $x\in X$ with respect to some fixed norm $\norm{\cdot}_X$. 
    Suppose that $O:C(X,Y) \times X \to \R^d$ is a continuous function, we endow the space $\mathcal O =\setB{O(f,\cdot):f\in \mathcal G}$ with the $L^2$ norm. 
    
    Assume that $O$ has a unique solution. Further, consider a hypothesis class $\mathcal F\subseteq\mathcal G$ and let $L: \mathcal F \to [0,\infty)$ be a function which satisfies
    \[
    \norm{f(x)-f(z)}_{Y} \leq  L(f)\norm{x-z}_X
    \]
    for all $x,z\in X$ and each $f\in \mathcal F$, where $\norm{\cdot}_Y$ is some fixed norm on  $Y$. That is, $L(f)$ is a Lipschitz constant of $f\in \mathcal F$. Let $L_{\max}>0$ be some fixed number and let $x_0 \in X$, $y_0\in Y$ be some points.
    
    Define the training error $E_{T,n}^{MSE}:\mathcal F \to \R$ by
    \begin{align*}
        E_{T,n}^{MSE}(f)
        =&~\frac{1}{n} \sum_{k=1}^n{\norm{O(f,a_k)}_{2}^2}
        \\&+\max\setB{0,L(f)-L_{\max}}+\norm{f(x_0)-y_0}_Y
    \end{align*}
    for $n\in \N$ and $\norm{\cdot}_2$ is the $2$-norm on $\R^d$. Then $(\mathcal F,E_{T,n}^{MSE})$ soundly approximates the solution to $O$.
\end{theorem}
We suspect the extra condition on the dense set, $\setB{a_k:k\in \N}$, in \cref{thm:MSE sound equtaion solve} is likely to hold in practice. 
It is somewhat difficult to even contrive of a dense set which does not satisfy the condition.

\section{Convergence Under Random Collocation}\label{sec:prob}

So far the convergence analysis has assumed deterministic training locations. In practice, PINNs are typically trained using randomly sampled collocation points. We therefore ask whether a small empirical residual still implies a small equation error under random sampling. The answer is affirmative under the same compactness assumption: convergence of the empirical residual in probability implies convergence in the solution space.

Let \(x_1,x_2,\dots\) be independent samples drawn from a probability measure \(\mu\) on \(X\). For \(p>0\), define the empirical residual
\[
E_{\text{test},p,n}(f)=\frac{1}{n}\sum_{k=1}^n \|O(f,x_k)\|_{\mathbb R^d}^p .
\]
This quantity serves as a stochastic approximation of the equation error \(E_{\mathcal O}(f)=\|O(f,\cdot)\|_{\mathcal O}\).

We say the probabilistic sound approximation property holds if, for any sequence \(\{f_n\}_{n\in\mathbb N}\subset \mathcal F\),
\[
\operatorname*{plim}_{n\to\infty} E_{test,p,n}(f_n)=0
\quad\Rightarrow\quad
f_n \to g \text{ in } \mathcal G,
\]
where $\operatorname*{plim}$ is the limit in probability (see Appendix \ref{sec:app:test_error}) and \(g\) is the unique solution satisfying \(O(g,\cdot)=0\).

The above condition states that vanishing empirical residual implies convergence to the true solution in probability. The following theorem shows that this property holds under the same compactness assumption used in the deterministic setting.

\begin{theorem}[Random collocation convergence, informal]
Let \(x_1,x_2,\dots\) be i.i.d.\ samples drawn from a probability measure \(\mu\) on \(X\).  
Assume \(E_{\mathcal O}(f)=\|O(f,\cdot)\|_{\mathcal O}\) is continuous and the sequence \(\{f_n\}_{n\in\mathbb N}\) lies in a compact subset of \(\mathcal G\). 
\correction{Assume further that every non-empty open set in $X$ has strictly positive measure under $\mu$.}
Then
\[
\operatorname*{plim}_{n\to\infty} E_{\mathrm{test},p,n}(f_n)=0
\;\Rightarrow\;
\lim_{n\to\infty}E_{\mathcal O}(f_n)=0 .
\]
Consequently, if \(O\) has a unique solution \(g\), then \(f_n\to g\) in \(\mathcal G\).
\end{theorem}


A more rigorous statement and proof and additional discussions can be found in Appendix \ref{sec:app:test_error}.

\section{Error Estimates and Verified Bounds}\label{sec:error_estimate}

\begin{table*}[ht!]
    \centering
    \begin{tabular}{cccccc}
        \toprule
        & \multicolumn{5}{c}{Stiffness Parameter} \\
        \cmidrule(lr){2-6}
        & $\mu=1$ & $\mu=2$ & $\mu=3$ & $\mu=4$ & $\mu=5$\\
        \midrule
        Reference error & 3.11E-4 & 2.06E-4 & 1.91E-4 & 7.0E-4 & 1.22E-3  \\
        Generalization bound & $\begin{bmatrix}
            1.66\text{E-2}\\ 1.69\text{E-2}
        \end{bmatrix}$ & $\begin{bmatrix}
            1.62\text{E-2}\\ 2.96\text{E-2}
        \end{bmatrix}$ & $\begin{bmatrix}
            2.36\text{E-2}\\ 6.94\text{E-2}
        \end{bmatrix}$ & $\begin{bmatrix}
            3.06\text{E-2}\\ 1.28\text{E-1}
        \end{bmatrix}$ &  $\begin{bmatrix}
            4.13\text{E-2}\\ 2.32\text{E-1}
        \end{bmatrix}$ \\
        \bottomrule
    \end{tabular}
    \caption{Van der Pol equation on $[0,1]$ for varying stiffness parameters $\mu$. We report the reference error (computed against a numerical solution based on RK4) and the certified solution bound obtained via Proposition~\ref{prop:ivp} using dReal. The certified bounds are conservative, with conservativeness increasing as stiffness grows.}
    \label{table:verifying_van_der_pol}
    \vspace{-0.2in}
\end{table*}

Sections \ref{sec:training_error} and \ref{sec:prob} establish qualitative convergence under compactness assumptions. We now derive quantitative generalization error estimates in the solution space that make these guarantees explicit. The resulting bounds isolate three ingredients: operator stability, residual control, and structural compactness of the hypothesis class.

\begin{theorem}[Generalization Error Estimate]\label{thm:generalization}
    Assume $X\subseteq \R^{d_X}$ is compact, $\Gc=C(X,Y)$ is equipped with the uniform norm, and $O$ admits a unique solution $g\in\Gc$. Assume the operator $O$ satisfies a stability condition: there exists $C_0>0$ such that
    \[\|f-g\|_\infty\leq C_0\|O(f,\cdot)\|_\Oc,\quad\text{for all }f\in\Gc,\]
    Let $\{a_k\}_{k=1}^n \subset X$ denote a sampling set,
    and define the fill distance \[\delta_n=\sup_{x \in X}\min_{k=1,\dots,n}\|x - a_k\|.\]
    
    Assume further that $f$ lies in a compact subset $\Fc \subset \Gc$. Then there exists a modulus of continuity $\omega_{\Fc}(\cdot)$, depending only on $\Fc$ (see Theorem \ref{thm:AA}), such that for all $f \in \Fc$, 
    \[|f(x) - f(y)| \le \omega_{\Fc}(\|x-y\|),
    ~\text{with }~\omega_{\Fc}(r) \to 0 \text{ as } r \to 0.\]
    
    Furthermore, there exists a constant $C>0$, depending only on $C_0$ and $\Fc$, such that
    \begin{align}
    \label{eq:generalization_representative}
    \|f - g\|_\infty\le C \Big(
    &\max_{k=1,\dots,n}\|\mathcal O(f,a_k)\| 
    +\omega_{\Fc}(\delta_n)\Big).
    \end{align}
\end{theorem}

\begin{remark}[Choice of norm]
\cref{thm:generalization} is stated in the uniform norm on $C(X,Y)$. Extending the result to other function-space norms requires corresponding compactness assumptions and stability estimates in those norms. In general, compactness properties and operator stability conditions depend sensitively on the chosen topology, which significantly complicates the analysis.

For $L_p$ norms with $1 \le p < \infty$ and $X$ of finite measure, a direct estimate follows from the uniform bound:
\[\|f-g\|_{L_p(X)}\le\lambda(X)^{1/p} \|f-g\|_{\infty},\]
\correction{where $\lambda$ is the Lebesgue measure defined on $\R^{d_X}$.}
Thus, the uniform generalization estimate immediately implies an $L_p$ bound, although potentially with conservative constants. For clarity and generality, we therefore formulate our main result in the uniform norm.
\end{remark}



We provide the proof in \cref{sec:app:thm4_proof}.
This bound decomposes into two components: (a) empirical residual error $\max_{k=1,...,n}\|O(f,a_k)\|_Y$, (b) 
compactness term $\omega_{\Fc}(\delta_n)$, which vanishes as the collocation points become dense.
Thus, provided the hypothesis class $\Hc$ is compact (via verification), and the collocation points are sufficiently concentrated, small residual error implies small solution error.


In practice, residual norms can be certified using formal verification techniques. Given a trained neural approximation $f$, one may compute a certified upper bound
\[\|\mathcal O(f,\cdot)\|_{\mathcal O}\le R_{cert},\]
where the bound is valid over the entire domain $X$. Such certificates can be obtained using bound-propagation methods, interval analysis, or SMT-based verification tools (see Section \ref{sec:related}). 

In this fully verified setting, residual control is no longer restricted to sampled collocation points. Consequently, the sampling density term $\omega_{\Fc}(\delta_n)$ in \eqref{eq:generalization_representative} vanishes, as the residual bound holds uniformly over $X$. The solution error estimate therefore simplifies to
\[\|f-g\|_\infty\le C_0\, R_{cert}.\]
Thus, uncertainty quantification for neural PDE solvers reduces to two equation-dependent ingredients: (a) A stability constant $C_0$ for the underlying PDE operator, (b) A certified uniform bound on the residual norm.

The constants appearing in these bounds depend on structural properties of the specific equation, domain geometry, and norm choice. Appendix \ref{app:case_studies} instantiates this abstract framework for representative elliptic, parabolic, and hyperbolic problems, deriving explicit stability constants and corresponding error estimates.

\section{Numerical examples}\label{sec:numerical}

\begin{table}
    \centering
    \label{your-table-label}
    \resizebox{\linewidth}{!}{
    \begin{tabular}{@{}lllll@{}} 
        \toprule
        PDE &
        \makecell{Boundary\\error} & \makecell{Residual\\error} & \makecell{Generalization\\bound} & \makecell{Reference\\error} \\ 
        \midrule
        2D Poisson & 1.74E-5 & 3.21E-2 & 4.03E-3 & 1.80E-3 \\ 
        (Prop. \ref{prop:elliptic}) & (20.51) & (2705s) & N/A & (1604s) \\ 
        \midrule
        Allen-Cahn & 2.28E-2 & 6.59E-4 & 4.09E-2 & 2.19E-2\\ 
        (Prop. \ref{prop:allen-cahn}) & (2691s) & (320s) & N/A & N/A\\
        \midrule
        Elasticity & 5.08E-3 & 7.22E-4 & 1.55E-2 & 1.06E-2\\ 
        (Prop. \ref{prop:elasticity}) & (4783s) & (195s) & N/A & N/A\\
        \bottomrule
    \end{tabular}}
    \caption{Verification of generalization error for 2D Poisson equation, Allen-Cahn equation, and linear elasticity equation: all errors are verified using autoLiRPA with the generalization error calculated using Propositions \ref{prop:elliptic}, \ref{prop:allen-cahn}, \ref{prop:elasticity}.
    }
    \label{table:verifying_3_examples}
    \vspace{-0.2in}
\end{table}

\begin{table*}[ht]
    \centering
    \begin{tabular}{@{}lllllll@{}} 
        \toprule
        \makecell{$\varepsilon|_{\partial\Omega}$ \\(certified)} & \makecell{$\varepsilon_t|_{\partial\Omega}$ \\(certified)} & \makecell{$\varepsilon(0,\cdot)$ \\(certified)} & \makecell{Residual\\bound} & \makecell{Generalization\\bound (Prop. \ref{prop:parabolic})} & \makecell{Generalization\\bound (Prop. \ref{prop:parabolic2})} & \makecell{Reference\\error} \\ 
        \midrule
        1.83E-4 & 2.24E-4 & 2.90E-4 & 3.16E-3 & 3.45E-3 & 7.45E-3 & 2.66E-3 \\ 
        (10.61s) & (10.32s) & (9.16s) & (2042s) & N/A & N/A & (951s)\\
        \bottomrule
    \end{tabular}
    \caption{Verification of generalization error for 1D heat equation. Let $\varepsilon:=u-u_{true}$ denote the errors corresponding to boundary conditions ($\partial\Omega$) and initial condition ($t=0$). All errors are certified under the uniform norm using autoLiRPA with the generalization error calculated using Propositions \ref{prop:parabolic} (uniform) and \ref{prop:parabolic2} ($L_2$). The reference error is reported in the $L_2$ norm.}
    \label{tab:heat1d}
    \vspace{-0.1in}
\end{table*}

\begin{table*}[ht]
    \centering
    \begin{tabular}{@{}lllllllll@{}} 
        \toprule
        \makecell{$\varepsilon|_{\partial\Omega}$ \\(certified)} & \makecell{$\varepsilon_t|_{\partial\Omega}$ \\(certified)} & \makecell{$\varepsilon_{tt}|_{\partial\Omega}$ \\(certified)} & \makecell{$\varepsilon(0,\cdot)$ \\(certified)} & \makecell{$\nabla\varepsilon(0,\cdot)$ \\(certified)} & \makecell{$\nabla\varepsilon\cdot\hat{n}|_{\partial\Omega}$ \\(certified)} & \makecell{Residual\\bound} & \makecell{Generalization\\bound (Prop. \ref{prop:hyperbolic})} & \makecell{Reference\\error} \\ 
        \midrule
        1.44E-5 & 2.89E-5 & 3.06E-5 & 2.28E-5 & 8.89E-5 & 3.38E-5 & 7.19E-5 & 4.65E-4 & 2.51E-4 \\ 
        (29.40s) & (39.33) & (60.82s) & (21.62s) & (8.40s) & (15.71s) & (2064s) & N/A & (1133s)\\
        \bottomrule
    \end{tabular}
    \caption{Verification of generalization error for 1D wave equation. Let $\varepsilon:=u-u_{true}$ denote the errors corresponding to boundary conditions ($\partial\Omega$) and initial condition ($t=0$). All errors are verified using autoLiRPA with the generalization error calculated using Proposition \ref{prop:hyperbolic}. }
    \label{tab:wave1d}
    \vspace{-0.2in}
\end{table*}

We now describe the common certification pipeline used in all examples below. Each equation is treated as an instantiation of the abstract generalization framework developed in Section \ref{sec:error_estimate}.
For a given equation, the procedure consists of four steps:
\paragraph{1. Training.} 
    We train a neural approximation $f$ using PINNs \citep{raissi2019physics,lagaris1998artificial} or ELMs \citep{huang2006extreme,chen2022bridging,dong2021local} depending on the problem. 
    We note that PINNs are better at solving nonlinear PDEs, and ELMs excel in obtaining high-accuracy solutions for low-dimensional systems. Recent work has also shown its potential for solving high-dimensional PDEs \citep{wang2024extreme}. 
    Training details are provided in Appendix~\ref{sec:app:implementation}.
\paragraph{2. Error Certification.} 
    Using formal verification tools such as dReal \citep{gao2013dreal} and autoLiRPA \citep{xu2020automatic}, we compute a certified upper bound $\|O(f,\cdot)\|_{O} \le R_{cert},$ valid over the entire domain.
    While dReal requires the user to specify an error tolerance threshold and search the domain for counterexamples, autoLiRPA requires the user to specify a subdomain and compute the error bound. We also post the time (in seconds) needed to perform certification for each error term.
\paragraph{3. Generalization Bound.} 
    Appendix~\ref{app:compactness_validation} empirically verifies compactness assumptions in ELMs. The Lipschitz bound $L \le \|\beta\| \|W\|$ remains stable or decreases as the network width increases.
    We apply Theorem \ref{thm:generalization} together with equation-specific stability estimates derived in Appendix \ref{app:case_studies} to obtain a solution-space error bound
    $\|f-g\|_\Gc\le C_0 R_{cert}.$
\paragraph{4. Reference Comparison.} For validation purposes only, we compare the certified solution bound against a reference solution, either computed analytically or via numerical solvers. The certified bound is obtained without access to the true solution.

\subsection{Ordinary Differential Equation}

\begin{example}[Van der Pol equation]
Consider the Van der Pol system
\begin{equation}
\begin{aligned}
\dot{x}_1 &= x_2, \\
\dot{x}_2 &= -x_1 + \mu (1 - x_1^2) x_2, \\
x(0) &= [1.0, 0.0]^T,
\end{aligned}
\end{equation}
where $\mu > 0$ controls the stiffness.
\end{example}

We train a neural approximation on $[0,1]$ using a
two-layer network with 10 neurons per layer.
Full training details are provided in
Appendix~\ref{sec:app:implementation}.
Residual bounds are certified using dReal. Applying Proposition \ref{prop:ivp} (see Remark \ref{rem:validated_sol_ivp}) in Appendix \ref{app:case_studies} yields a certified solution bound over $[0,1]$. 
Certified residual and solution bounds are reported in Table~\ref{table:verifying_van_der_pol}.

\subsection{Elliptic PDE}

\begin{example}[Poisson's equation]
Consider a two-dimensional Poisson's equation 
$\Delta u = f, $
on the domain $D=(0,1)\times (0,1)$ with boundary conditions $u\equiv 0$ on $\partial\Omega$ and the source term $f(x,y) = \sin(\pi x)\sin(\pi y)$.  
This corresponds to the true solution
\[
u_{true} = -\frac{1}{2\pi^2}\sin(\pi x)\sin(\pi y).
\]
\end{example}


We approximated the solution using ELMs and report the verification results in Tables \ref{table:verifying_3_examples}.
Residual error bounds and boundary error bounds are certified using autoLiRPA. Applying Proposition \ref{prop:elliptic} (see remark \ref{rem:est2}) in Appendix \ref{app:case_studies} yields a certified solution bound over $\Omega$. Our results empirically show that the generalization bound is greater than the true error, but does not significantly overestimate it.

\subsection{Parabolic PDE}

\begin{example}[Heat equation]\label{ex:heat}
Consider an $n$-dimensional heat equation defined on $\Omega=[0,1]^n$
\begin{equation}
    \begin{aligned}
        u_{t} - \alpha\Delta u = 0\quad &\text{in }\Omega\times [0,T], \\
        u(x,0) = \prod_{i=1}^{n}\sin(\pi x_i) \quad &\text{on }\Omega,\\
        u \equiv 0 \quad &\text{on }\partial\Omega\times (0,T),
    \end{aligned}
\end{equation}
that has the following analytic solution
$$
u_{true}\correction{(x,t)} = e^{-n\alpha\pi^2t}\prod_{i=1}^{n}\sin(\pi x_i).
$$   
\end{example}

We set $\alpha=0.1$ in our experiments and approximate the solution using ELMs, plotted in Figure \ref{fig:heat_wave1D}. Uniform residual, initial, and boundary bounds are certified using autoLiRPA. Proposition~\ref{prop:parabolic} provides a uniform generalization error, while Proposition~\ref{prop:parabolic2} yields an $L_2$-based bound.
Table~\ref{tab:heat1d} reports certified residual bounds with $n=1$, the resulting certified solution bounds under both propositions, and the reference error computed as the uniform norm difference between the neural and analytic solutions. We post additional results in \cref{tab:heat2}, where we compute $L_2$ errors of each source of error to provide a tighter bound based on \cref{prop:parabolic}.

\subsection{Hyperbolic PDE}

\begin{example}[Wave equation]
Consider a 1-dimensional wave equation defined on $\Omega=[0,1]$
\begin{equation}
    \begin{aligned}
        u_{tt} - \alpha\Delta u + cu = 0\quad &\text{in }\Omega\times [0,T], \\
        u(x,0) = \sin(\pi x), u_t(x,0) = 0 \quad &\text{on }\Omega,\\
        u \equiv 0 \quad &\text{on }\partial\Omega\times (0,T),
    \end{aligned}
\end{equation}
that has the following analytic solution
$$
u_{true}(x,t) = \cos(t\sqrt{\alpha\pi^2+c})\sin(\pi x).
$$   
\end{example}
We set $\alpha=0.1,c=0.2$ in our experiments and approximate the solution using ELMs, plotted in Figure \ref{fig:heat_wave1D}. Uniform residual, initial, and boundary bounds are certified with autoLiRPA. We provide further results with a tighter $L_2$-based estimate using Proposition \ref{prop:hyperbolic} in Table \ref{tab:wave1d_2}.

\subsection{Nonlinear PDE}

\begin{example}[Allen-Cahn equation]
Consider a 1D Allen-Cahn equation defined on $\Omega=[-1,1]$ with periodic boundary conditions
\begin{equation}
    \begin{aligned}
        u_{t} + \rho u(u^2-1) = \nu u_{xx}\quad &\text{in }\Omega\times [0,T], \\
        u(x,0) = (1-x^2)^2\cos(2\pi x) \quad &\text{on }\Omega.
    \end{aligned}
\end{equation}
We set $\rho=1$ and $\nu=0.01$ in our example.
This problem does not admit a closed-form analytic solution. For reference, we compare against a numerical solution computed via Crank-Nicolson \citep{Crank_Nicolson_1947}.
\end{example}
We approximate the solution using PINNs. Uniform residual and initial-condition bounds are certified with autoLiRPA. In the neural network architecture, we modify the input space by mapping $(x,t)$ to $(\sin(\pi x),\cos(\pi x),t)$, which forces periodicity in the solution. Applying Proposition~\ref{prop:allen-cahn} yields a certified solution bound in the $L_2$ norm.
Our results are posted in Table \ref{table:verifying_3_examples}.


\subsection{Multi-Output PDE}

\correction{The proposed framework naturally extends to systems of PDEs. Error certification can be performed component-wise, and verification cost scales linearly with output dimension.}

\begin{example}[Linear Elasticity equation]
Consider a 2D Linear Elasticity equation defined on $\Omega=[0,1]^2$
\begin{equation}
\begin{split}
    -\nabla\cdot\sigma(u)=f\quad \text{in }\Omega,\quad
    u=0\quad &\text{in }\partial\Omega,
\end{split}
\end{equation}
where $f:\Omega\rightarrow\R^n$ is a body force and
\begin{equation*}
\sigma(u)
=
2\mu \varepsilon(u)
+
\lambda \operatorname{tr}(\varepsilon(u))I
\end{equation*}
is the Cauchy stress tensor. Here
\begin{equation*}
\varepsilon(u)
=
\frac12\left(\nabla u+\nabla u^\top\right)
\end{equation*}
denotes the strain tensor, while $\mu>0$ and $\lambda\ge 0$ are the Lamé parameters. The source term $f$ is given by
\[
f(x,y)=\begin{bmatrix}
    2\mu\pi^2s(x,y)-(\lambda+\mu)\pi^2(c(x,y)^2-s(x,y)^2)\\
    2\mu\pi^2s(x,y)-(\lambda+\mu)\pi^2(c(x,y)^2-s(x,y)^2)\\
\end{bmatrix},
\]
where $s(x,y)=\sin(\pi x)\sin(\pi y)$ and $c(x,y)=\cos(\pi x)\cos(\pi y)$. The analytical solution is
\begin{equation*}
    u_{true}(x,y)=\begin{bmatrix}s(x,y)&c(x,y)\end{bmatrix}^T.
\end{equation*}
\end{example}
We set $\lambda=1$ and $\mu=1$ in our example and approximate the solution using ELMs. Residual and boundary error bounds are certified using autoLiRPA. Our results are reported in Table \ref{table:verifying_3_examples}.

\section{Conclusions}

We develop a generalization stability framework for neural differential equation solvers. Under compactness and operator stability assumptions, certified residual bounds can be converted into explicit solution-space 
error guarantees. This provides a principled pathway to generalization bounds on the true solution.
We provide a certification pipeline: residual bounds verified by formal methods can be systematically converted into uniform or $L_2$ solution-space guarantees.
Several directions remain open. A deeper understanding of compactness mechanisms in modern neural architectures, sharper stability constants for specific equation classes, and extensions to operator-learning or stochastic PDE settings would further strengthen the theoretical foundations of certified scientific machine learning.

\newpage 

\bibliography{biblo}

\newpage

\onecolumn

\appendix

\section{Compactness in the uniform norm}\label{sec:unifrom norm}

\begin{theorem}[Arzelà–Ascoli theorem \citep{arzela1895sulle}]\label{thm:AA}
    Suppose that $\mathcal F$ is a set of continuous functions from a compact metric space set $X$ with metric $\dist_X$ to a metric space $Y$ with metric $\dist_Y$. The following are equivalent:
    \begin{enumerate}
        \item $\closure{\mathcal F}$ is compact with respect to the uniform metric.
        \item The following hold:
        \begin{enumerate}
            \item  $\mathcal F$ is pointwise relatively compact. i.e for all $x\in X$ the set $\closure{\setB{f(x) : f \in \mathcal F}}$ is compact.
            \item $\mathcal F$ is uniformly equicontinuous.  i.e for all $\epsilon>0$ there is a $\delta >0$ such that for all $f \in \mathcal F$ and all $x\in X$ we have $f(\B[X]{\delta}{x})\subseteq \B[Y]{\epsilon}{f(x)}$
        \end{enumerate}
        \item The following hold:
        \begin{enumerate}
            \item  $\mathcal F$ is pointwise relatively compact. i.e for all $x\in X$ the set $\closure{\setB{f(x) : f \in \mathcal F}}$ is compact.
            \item $\mathcal F$ is equicontinuous.  i.e for all $\epsilon>0$ and for all $x\in X$  there is a $\delta >0$ such that for all $f \in \mathcal F$ we have $f(\B[X]{\delta}{x})\subseteq \B[Y]{\epsilon}{f(x)}$ 
        \end{enumerate}
    \end{enumerate}
\end{theorem}
\correction{\Cref{thm:AA} is a classic theorem in analysis. It characterizes the compact subsets of $C(X,Y)$ as the uniformly equicontinuous and pointwise relatively compact sets. }
\correction{
Essentially, uniform equicontinuity of $\mathcal F$ can be thought of as: there is a $L_{\max}$ such that every $f\in \mathcal F$ we have $\rho_X(f(x),f(y))\leq L_{\max}\rho_Y(x,y)$ for every $x,y\in X$. That is, elements of $\mathcal F$ have a uniformly bounded Lipschitz constant; sometimes this is called equi-Lipschitz. 
}

\correction{
When the space $Y$ is $\R^d$ endowed with the Euclidean norm, the sets $\mathcal F$ which are pointwise relatively compact are the pointwise bounded sets, that is: for all $x\in X$ $\sup_{f\in \mathcal F} \norm{f(x)}< \infty$. If the space $X$ is also connected, we can weaken this condition even more to: \emph{for some} $x\in X$  $\sup_{f\in \mathcal F} \norm{f(x)}< \infty$. The following result proves this for an even weaker condition.}

\begin{lemma}[Pined at a point lemma]\label{thm:localGraphConvergece}
    Suppose that $X$ is a connected compact subset of $\R^{d_{X}}$ containing at least two points and let $Y=\R^{d_Y}$.  Let $\mathcal F$ be a set of continuous functions from $X$ to $Y$. 
    Assume that there are compact sets $K\subseteq X$ and $J\subseteq Y$ such that every $f\in \mathcal F$ has $f(K)\cap J\neq \emptyset$. For example, $K=\setB{x}$ and $J=\closure{\B[Y]{r}{y}}$ for some particular points $x\in X$, $y\in Y$ and $r>0$ is a realistic assumption.
    
    If $\mathcal F$ is uniformly equicontinuous then, $\mathcal F$ is uniformly bounded and thus $\closure{\mathcal F}$ is compact. 
\end{lemma}
\begin{proof}
    Since  $\mathcal F$ is uniformly equicontinuous we have that 
    \[
    \forall \epsilon >0 \, \exists \delta(\epsilon)=\delta >0 \, \forall x\in X \forall f\in \mathcal F \, \text{ we have } f\round{\B[X]{\delta}{x}} \subseteq \B[Y]{\epsilon}{f(x)}
    \]
    
    Define
    \[
    \omega(\epsilon) = \sup\setB{\norm{f(x)-f(z)}_Y:f \in \mathcal F, x,z\in X, \norm{x-z}_X<\epsilon }.
    \]
        Note that $\omega$ is an non-decreasing function.
    One can argue that $\omega(\epsilon) < \infty$ for all $\epsilon>0$. However this proof is laborious, as an outline one can first show that for all $\eta>0$ there is a $N\in \N$ with for all $x,z\in X$ there is sequence satisfying $\norm{x_n-x_{n+1}} <\eta$ with $x_0 =x$ and $x_{N+1}=z$. Assuming this, one quickly sees for all $\epsilon >0$ we can pick $\eta=\delta(\epsilon)$ and we get, for all $f \in \mathcal F$, $x,z\in X, \norm{x-z}_X<\epsilon$
    \[
    \norm{f(x)-f(z)}\leq \sum_{n=0}^{N} \norm{f(x_n)-f(x_{n+1})} < \epsilon (N+1)
    \]
    and so $\omega(\epsilon) < \infty$ must be the case.
    
    Regardless, $\omega$ should be thought of as an ``inverse" to $\delta(\epsilon)$, for $\omega$ satisfies 
    \[
    \forall \epsilon >0 \, \forall x\in X \, \forall f\in \mathcal F \text{ we have } f\round{\B[X]{\epsilon}{x}} \subseteq \B[Y]{\omega(\epsilon)}{f(x)}
    \]
    or in other symbols
    \[
    \forall \epsilon >0 \, \forall x\in X \, \forall f\in \mathcal F \text{ we have } \norm{x-z}_X <\epsilon \implies \norm{f(x)-f(z)} < \omega(\epsilon).
    \]
    Let $x\in X$, $z\in J$ be arbitrary. Since $J$ is compact we know that it is bounded so let $M_J = \max_{y\in J}\norm{y}_Y$. We have, for all $f\in \mathcal F$
    \[
        \norm{f(x)}_Y \leq \norm{f(z)}_Y + \norm{f(x)-f(z)}_Y \leq M_J+\omega(\operatorname{diam}(X))
    \]
    where $\operatorname{diam}(X)=\sup_{x',z'\in X}\norm{x'-z'}_X<\infty$.
    The RHS is bounded and independent of $f$ and $x$. This shows that $\mathcal F$ is uniformly bounded.
\end{proof}

\correction{The condition with the compact sets $K$, $J$ a geometric interpretation. If we identify every function $f\in \mathcal F$ as its graph, $\operatorname{Gr(f)= \set{(x,f(x)): x\in X}} \subseteq X\times Y$, then the condition means that the graph of every $f\in \mathcal F$ goes through the compact ``rectangle''  $K\times J$.  }

\correction{We now connect this result to \cref{thm:sound equtaion solve} and the uniform compactness term, $\kappa(f)=\max\setB{0,L(f)-L_{\max}}+\norm{f(x_0)-y_0}$. We observe that if $\sup_{f\in \mathcal F} \kappa(f) < r <\infty$ then, (i) we see that $\mathcal F$ is equi-Lipschitz as $L(f)< L_{\max} +r $ for all $f\in \mathcal F$ and (ii) the $K$, $J$ condition is satisfied for $K=\set{x_0}$, $J=\closure{\B[Y]{r}{y_0}}$ since $\norm{f(x_0)-y_0} <r$. Therefore $\sup_{f\in \mathcal F} \kappa(f) <\infty$ implies that $\closure{\mathcal F}$ is compact. 
}

\section{Estimation of Lipschitz Constants}\label{sec:Estimation of Lipschitz Constants}
Throughout this section, vector norms are the \(\ell_\infty\) norms. If
\(A=(a_{ij})\) is a matrix, \(\lVert A\rVert_\infty\) denotes the induced
\(\ell_\infty\) operator norm,
\[
    \lVert A\rVert_\infty=\max_i\sum_j |a_{ij}|.
\]
For a map \(f:X\to\mathbb{R}^m\), write
\[
    \operatorname{Lip}_X(f)
    :=\sup_{x\neq y\in X}
    \frac{\lVert f(x)-f(y)\rVert_\infty}{\lVert x-y\rVert_\infty},
\]
with the convention that this supremum is \(0\) if \(X\) has at most one point.
When the domain is clear, write \(\operatorname{Lip}(f)\). For a scalar map
\(\phi:\mathbb{R}\to\mathbb{R}\), define its componentwise extension by
\[
    \widetilde\phi(z_1,\ldots,z_m):=(\phi(z_1),\ldots,\phi(z_m)).
\]

\begin{lemma}[Basic Lipschitz estimates]\label{lem:basic-lipschitz-estimates}
Let \(X\subset\mathbb{R}^n\) be nonempty, and let \(f,g:X\to\mathbb{R}^m\) be bounded
Lipschitz maps. Then
\[
    \operatorname{Lip}_X(f\odot g)
    \le
    \Bigl(\sup_{x\in X}\lVert f(x)\rVert_\infty\Bigr)\operatorname{Lip}_X(g)
    +
    \Bigl(\sup_{x\in X}\lVert g(x)\rVert_\infty\Bigr)\operatorname{Lip}_X(f).
\]
If \(h:X\to Y\) and \(q:Y\to Z\) are Lipschitz maps between metric spaces, then,
using the corresponding metrics,
\[
    \operatorname{Lip}(q\circ h)\le \operatorname{Lip}(q)\operatorname{Lip}(h).
\]
Finally, if \(f:X\to\mathbb{R}^m\) is Lipschitz and
\(\phi:\mathbb{R}\to\mathbb{R}\) is Lipschitz on an interval containing every
component of every point of \(f(X)\), with Lipschitz constant \(L_\phi\) on that
interval, then
\[
    \operatorname{Lip}_X(\widetilde\phi\circ f)
    \le L_\phi\operatorname{Lip}_X(f).
\]
\end{lemma}

\begin{proof}
For the product estimate, write
\[
    f(x)\odot g(x)-f(y)\odot g(y)
    =f(x)\odot\bigl(g(x)-g(y)\bigr)
    +\bigl(f(x)-f(y)\bigr)\odot g(y).
\]
Since \(\lVert u\odot v\rVert_\infty\le
\lVert u\rVert_\infty\lVert v\rVert_\infty\), division by
\(\lVert x-y\rVert_\infty\) and taking suprema gives the claim. The ordinary
composition estimate follows directly from the definitions. For scalar
componentwise composition,
\[
\begin{aligned}
    \lVert \widetilde\phi(f(x))-\widetilde\phi(f(y))\rVert_\infty
    &=\max_i |\phi(f_i(x))-\phi(f_i(y))|  \\
    &\le L_\phi\max_i |f_i(x)-f_i(y)|
    =L_\phi\lVert f(x)-f(y)\rVert_\infty .
\end{aligned}
\]
Taking suprema over \(x\neq y\) proves the estimate.
\end{proof}

\begin{proposition}[Derivative estimate for componentwise nonlinearities]
\label{prop:componentwise-derivative-estimate}
Let \(U\subset\mathbb{R}^n\) be open, let \(X\subset U\) be nonempty, and let
\(f:U\to\mathbb{R}^m\) be \(C^1\). Fix \(j\in\{1,\ldots,n\}\). Assume that
\(f|_X\) and \((\partial_j f)|_X\) are bounded and Lipschitz. Let
\(\phi:\mathbb{R}\to\mathbb{R}\) be \(C^1\), and assume that \(\phi'\) is
Lipschitz on an interval containing every component of every point of \(f(X)\),
with Lipschitz constant \(L_{\phi'}\) on that interval. Then
\[
    \partial_j(\widetilde\phi\circ f)(x)
    =\bigl(\widetilde{\phi'}(f(x))\bigr)\odot \partial_j f(x),
\]
and
\[
\begin{aligned}
    \operatorname{Lip}_X\bigl(\partial_j(\widetilde\phi\circ f)\bigr)
    \le{}&
    \Bigl(\sup_{x\in X}\lVert \widetilde{\phi'}(f(x))\rVert_\infty\Bigr)
    \operatorname{Lip}_X(\partial_j f) \\
    &+L_{\phi'}
    \Bigl(\sup_{x\in X}\lVert \partial_j f(x)\rVert_\infty\Bigr)
    \operatorname{Lip}_X(f).
\end{aligned}
\]
\end{proposition}

\begin{proof}
The displayed identity is the componentwise chain rule. Apply the product
estimate in Lemma~\ref{lem:basic-lipschitz-estimates} to
\(\widetilde{\phi'}\circ f\) and \(\partial_j f\). The only remaining input is
\[
    \operatorname{Lip}_X(\widetilde{\phi'}\circ f)
    \le L_{\phi'}\operatorname{Lip}_X(f),
\]
which is the scalar componentwise composition estimate from
Lemma~\ref{lem:basic-lipschitz-estimates}.
\end{proof}

Let
\[
    \rho(t):=t_+^2=\max\{0,t\}^2 .
\]
Then \(\rho\in C^1(\mathbb{R})\), \(\rho'(t)=2t_+\),
\(\operatorname{Lip}(\rho')=2\), and, for every \(M\ge 0\),
\[
    \operatorname{Lip}(\rho|_{[-M,M]})\le 2M,
    \qquad
    \sup_{|t|\le M}|\rho(t)|\le M^2,
    \qquad
    \sup_{|t|\le M}|\rho'(t)|\le 2M .
\]
Indeed, \(\rho'\) has values in \([0,2M]\) on \([-M,M]\), and
\(|\rho'(s)-\rho'(t)|=2|s_+-t_+|\le 2|s-t|\).

\begin{proposition}[Squared-ReLU after an affine map]
\label{prop:squared-relu-affine-estimates}
Let \(U\subset\mathbb{R}^d\) be open, let \(X\subset U\) be nonempty, and let
\(f:U\to\mathbb{R}^n\) be \(C^1\). Assume that \(f|_X\) and
\((\partial_j f)|_X\) are bounded and Lipschitz for each \(j=1,\ldots,d\). Let
\[
    c(z)=Az+b,
    \qquad A\in\mathbb{R}^{m\times n},
    \qquad b\in\mathbb{R}^m.
\]
Then
\[
    \sup_{x\in X}\lVert c(f(x))\rVert_\infty
    \le
    \lVert A\rVert_\infty\sup_{x\in X}\lVert f(x)\rVert_\infty
    +\lVert b\rVert_\infty,
\]
and
\[
\begin{aligned}
    \operatorname{Lip}_X(\widetilde\rho\circ c\circ f)
    \le{}&
    2\Bigl(\lVert A\rVert_\infty\sup_{x\in X}\lVert f(x)\rVert_\infty
    +\lVert b\rVert_\infty\Bigr)
    \lVert A\rVert_\infty\operatorname{Lip}_X(f).
\end{aligned}
\]
Moreover, for each \(j=1,\ldots,d\),
\[
\begin{aligned}
    \operatorname{Lip}_X\bigl(\partial_j(\widetilde\rho\circ c\circ f)\bigr)
    \le{}&
    2\Bigl(\lVert A\rVert_\infty\sup_{x\in X}\lVert f(x)\rVert_\infty
    +\lVert b\rVert_\infty\Bigr)
    \lVert A\rVert_\infty\operatorname{Lip}_X(\partial_j f) \\
    &+2\lVert A\rVert_\infty^2
    \Bigl(\sup_{x\in X}\lVert\partial_j f(x)\rVert_\infty\Bigr)
    \operatorname{Lip}_X(f).
\end{aligned}
\]
\end{proposition}

\begin{proof}
The supremum bound follows from
\[
    \lVert A f(x)+b\rVert_\infty
    \le \lVert A\rVert_\infty\lVert f(x)\rVert_\infty+\lVert b\rVert_\infty.
\]
Thus \(c\circ f\) takes values in the cube \([-M,M]^m\), where
\[
    M:=\lVert A\rVert_\infty\sup_{x\in X}\lVert f(x)\rVert_\infty
    +\lVert b\rVert_\infty .
\]
Since \(\rho\) is \(2M\)-Lipschitz on \([-M,M]\), and since
\[
    \operatorname{Lip}_X(c\circ f)
    \le \lVert A\rVert_\infty\operatorname{Lip}_X(f),
\]
the stated bound for \(\operatorname{Lip}_X(\widetilde\rho\circ c\circ f)\)
follows.

For the derivative bound, apply
Proposition~\ref{prop:componentwise-derivative-estimate} to \(c\circ f\). We
use
\[
    \sup_{x\in X}\lVert \widetilde{\rho'}(c(f(x)))\rVert_\infty\le 2M,
    \qquad
    \operatorname{Lip}_X(c\circ f)
    \le \lVert A\rVert_\infty\operatorname{Lip}_X(f),
\]
\[
    \partial_j(c\circ f)=A\partial_j f,
    \qquad
    \sup_{x\in X}\lVert A\partial_j f(x)\rVert_\infty
    \le \lVert A\rVert_\infty
    \sup_{x\in X}\lVert\partial_j f(x)\rVert_\infty,
\]
and
\[
    \operatorname{Lip}_X(A\partial_j f)
    \le \lVert A\rVert_\infty\operatorname{Lip}_X(\partial_j f).
\]
Since \(\operatorname{Lip}(\rho')=2\), substitution into
Proposition~\ref{prop:componentwise-derivative-estimate} gives the displayed
estimate.
\end{proof}

\begin{theorem}[Recursive bound for a squared-ReLU network]
\label{thm:squared-relu-network-lipschitz-gradient-bound}
Let \(X\subset\mathbb{R}^{d_X}\) be nonempty and bounded, and set \(d_0=d_X\).
For \(k=1,\ldots,K\), let
\[
    c_k(z)=A_kz+b_k,
    \qquad
    A_k\in\mathbb{R}^{d_k\times d_{k-1}},
    \qquad
    b_k\in\mathbb{R}^{d_k}.
\]
Define \(f_0:\mathbb{R}^{d_X}\to\mathbb{R}^{d_X}\) and
\(f_k:\mathbb{R}^{d_X}\to\mathbb{R}^{d_k}\) recursively by
\[
    f_0(x)=x,
    \qquad
    f_k(x)=\widetilde\rho\bigl(c_k(f_{k-1}(x))\bigr),
    \qquad k=1,\ldots,K.
\]
Define
\[
    R_0:=\sup_{x\in X}\lVert x\rVert_\infty,
    \qquad
    L_0:=1,
    \qquad
    S_0:=0,
\]
and, for \(k=1,\ldots,K\),
\[
    R_k:=\bigl(\lVert A_k\rVert_\infty R_{k-1}+\lVert b_k\rVert_\infty\bigr)^2,
\]
\[
    L_k:=2\bigl(\lVert A_k\rVert_\infty R_{k-1}+\lVert b_k\rVert_\infty\bigr)
    \lVert A_k\rVert_\infty L_{k-1},
\]
and
\[
    S_k:=2\bigl(\lVert A_k\rVert_\infty R_{k-1}+\lVert b_k\rVert_\infty\bigr)
    \lVert A_k\rVert_\infty S_{k-1}
    +2\lVert A_k\rVert_\infty^2L_{k-1}^2.
\]
Then, for every \(k=0,\ldots,K\),
\[
    \sup_{x\in X}\lVert f_k(x)\rVert_\infty\le R_k,
    \qquad
    \operatorname{Lip}_X(f_k)\le L_k.
\]
Moreover, for every input coordinate \(j=1,\ldots,d_X\),
\[
    \sup_{x\in X}\left\lVert\frac{\partial f_k}{\partial x_j}(x)\right\rVert_\infty
    \le L_k,
    \qquad
    \operatorname{Lip}_X\left(\frac{\partial f_k}{\partial x_j}\right)
    \le S_k.
\]
View \(\nabla f_K(x)\) as the \(d_X\)-tuple of its columns
\(\partial f_K/\partial x_j\), and equip
\(\mathbb{R}^{d_K}\times(\mathbb{R}^{d_K})^{d_X}\) with the product
\(\ell_\infty\) norm. Then
\[
    \operatorname{Lip}_X\bigl(x\mapsto (f_K(x),\nabla f_K(x))\bigr)
    \le \max\{L_K,S_K\}.
\]
\end{theorem}

\begin{proof}
The proof is by induction. For \(k=0\), \(f_0(x)=x\), so
\(\sup_{x\in X}\lVert f_0(x)\rVert_\infty=R_0\),
\(\operatorname{Lip}_X(f_0)\le 1=L_0\),
\(\sup_{x\in X}\lVert \partial f_0/\partial x_j\rVert_\infty=1=L_0\), and
\(\operatorname{Lip}_X(\partial f_0/\partial x_j)=0=S_0\).

Assume the stated estimates hold up to layer \(k-1\). Then
\[
    \sup_{x\in X}\lVert c_k(f_{k-1}(x))\rVert_\infty
    \le \lVert A_k\rVert_\infty R_{k-1}+\lVert b_k\rVert_\infty.
\]
Since \(|\rho(t)|\le
(\lVert A_k\rVert_\infty R_{k-1}+\lVert b_k\rVert_\infty)^2\) whenever
\(|t|\le \lVert A_k\rVert_\infty R_{k-1}+\lVert b_k\rVert_\infty\), this gives
\[
    \sup_{x\in X}\lVert f_k(x)\rVert_\infty\le R_k.
\]
On the interval
\([ -\lVert A_k\rVert_\infty R_{k-1}-\lVert b_k\rVert_\infty,
\lVert A_k\rVert_\infty R_{k-1}+\lVert b_k\rVert_\infty]\), the function
\(\rho\) is Lipschitz with constant at most
\(2(\lVert A_k\rVert_\infty R_{k-1}+\lVert b_k\rVert_\infty)\). Therefore
\[
\begin{aligned}
    \operatorname{Lip}_X(f_k)
    &\le
    2\bigl(\lVert A_k\rVert_\infty R_{k-1}+\lVert b_k\rVert_\infty\bigr)
    \operatorname{Lip}_X(c_k\circ f_{k-1}) \\
    &\le
    2\bigl(\lVert A_k\rVert_\infty R_{k-1}+\lVert b_k\rVert_\infty\bigr)
    \lVert A_k\rVert_\infty\operatorname{Lip}_X(f_{k-1})
    \le L_k.
\end{aligned}
\]

It remains to control derivatives. Fix \(j\in\{1,\ldots,d_X\}\), and write
\(D_{k,j}:=\partial f_k/\partial x_j\). The chain rule gives
\[
    D_{k,j}(x)
    =\bigl(\widetilde{\rho'}(c_k(f_{k-1}(x)))\bigr)
    \odot A_kD_{k-1,j}(x).
\]
The preceding bound on \(c_k\circ f_{k-1}\), together with
\(\rho'(t)=2t_+\), gives
\[
    \sup_{x\in X}\lVert \widetilde{\rho'}(c_k(f_{k-1}(x)))\rVert_\infty
    \le
    2\bigl(\lVert A_k\rVert_\infty R_{k-1}+\lVert b_k\rVert_\infty\bigr).
\]
Thus
\[
\begin{aligned}
    \sup_{x\in X}\lVert D_{k,j}(x)\rVert_\infty
    &\le
    2\bigl(\lVert A_k\rVert_\infty R_{k-1}+\lVert b_k\rVert_\infty\bigr)
    \lVert A_k\rVert_\infty
    \sup_{x\in X}\lVert D_{k-1,j}(x)\rVert_\infty \\
    &\le L_k.
\end{aligned}
\]
For the Lipschitz estimate on \(D_{k,j}\), use the product estimate from
Lemma~\ref{lem:basic-lipschitz-estimates}. We have
\[
    \operatorname{Lip}_X(A_kD_{k-1,j})
    \le \lVert A_k\rVert_\infty S_{k-1},
\]
and
\[
    \operatorname{Lip}_X\bigl(\widetilde{\rho'}\circ c_k\circ f_{k-1}\bigr)
    \le 2\lVert A_k\rVert_\infty L_{k-1}.
\]
Also,
\[
    \sup_{x\in X}\lVert A_kD_{k-1,j}(x)\rVert_\infty
    \le \lVert A_k\rVert_\infty L_{k-1}.
\]
Consequently,
\[
\begin{aligned}
    \operatorname{Lip}_X(D_{k,j})
    &\le
    2\bigl(\lVert A_k\rVert_\infty R_{k-1}+\lVert b_k\rVert_\infty\bigr)
    \lVert A_k\rVert_\infty S_{k-1} \\
    &\qquad
    +\lVert A_k\rVert_\infty L_{k-1}\cdot
    2\lVert A_k\rVert_\infty L_{k-1}
    =S_k.
\end{aligned}
\]
This completes the induction.

Finally, under the stated product norm,
\[
\begin{aligned}
    &\left\lVert (f_K(x),\nabla f_K(x))-(f_K(y),\nabla f_K(y))\right\rVert_\infty \\
    &\qquad=
    \max\left\{
    \lVert f_K(x)-f_K(y)\rVert_\infty,
    \max_{1\le j\le d_X}\lVert D_{K,j}(x)-D_{K,j}(y)\rVert_\infty
    \right\}.
\end{aligned}
\]
Dividing by \(\lVert x-y\rVert_\infty\), taking the supremum over \(x\neq y\),
and applying the estimates above yields
\[
    \operatorname{Lip}_X\bigl(x\mapsto (f_K(x),\nabla f_K(x))\bigr)
    \le \max\{L_K,S_K\}.
\]
\end{proof}

\subsection{Compactness Validation}
\label{app:compactness_validation}

Section \ref{sec:training_error} establishes that compactness of the hypothesis sequence is a key structural assumption underlying our convergence analysis. In particular, by Lemma \ref{thm:localGraphConvergece}, relative compactness in $C(X,Y)$ follows if the sequence shares a uniform modulus of continuity and is bounded at a point.

We empirically examine this assumption for the heat equation in Example \ref{ex:heat}. We consider ELM-based neural approximations of the form \begin{equation}\label{eq:elm} f(x) = \beta^T \sigma(Wx+b), \end{equation} where hidden-layer parameters are randomized and output weights are trained via least squares (see Appendix~\ref{sec:app:implementation} for details).

For such networks, a computable Lipschitz
over-approximation is given by
$L \le \|\beta\| \|W\|.$
Table~\ref{table:heat_lip_constant} reports these Lipschitz estimates for varying network widths and training configurations.

Empirically, although $\|W\|$ increases with network width, the product $\|\beta\|\|W\|$ remains stable and often decreases. This behavior is consistent with the compactness assumption required by our theoretical framework. A rigorous characterization of this phenomenon for random feature networks remains an interesting direction for future work.

\begin{table}[ht]
    \centering
    \label{your-table-label}
    \begin{tabular}{@{}llll@{}} 
        \toprule
        \makecell{Width\\($m$)} & \makecell{Training\\points ($N$)} & \makecell{Lipschitz constant \\ ($L=\norm{\beta}\norm{W}$)} & \makecell{Test\\error} \\ 
        \midrule
    800 & 3,000 & 5,357 & 3.53E-7 \\ 800 & 30,000 & 3,901 & 1.02E-7 \\      
    1,600 & 3,000 & 720 & 6.87E-9 \\
    1,600 & 30,000 & 582 & 4.39E-8 \\
    3,200 & 3,000 & 321 & 2.17E-9 \\
    3,200 & 30,000 & 304 & 1.22E-8 \\
    6,400 & 3,000 & 204 & 1.19E-9 \\
    6,400 & 30,000 & 198  & 1.51E-8 \\  
\bottomrule
    \end{tabular}
    \caption{Estimations of the Lipschitz constants of ELMs to the heat equation in Example \ref{ex:heat} with $n=3$, for different numbers of collocation (training) points ($N$) and numbers of hidden units ($m$), i.e., network width. We use $L=\norm{\beta}\norm{W}$ as an over-approximation of the true Lipschitz constants. The test error is the maximum true error computed at $2N$ test points.}
    \label{table:heat_lip_constant}
\end{table}

\section{Proofs on Convergence Results}
\label{app:convergence}

\textbf{Proposition \ref{thm:no training compact}.}
    Let $O:\mathcal G \times X \to \R^{d_O}$ be an equation, with $E_{\mathcal O}(f)=\norm{O(f,\cdot)}_{\mathcal O}$ continuous. 
    If $O$ has a unique solution $g\in \mathcal G$ then,
    \[
    \closure{\setB{f_n:n\in \N}} \text{ is compact and } \lim_{n\to \infty} E_{\mathcal O}(f_n) =0 \iff \lim_{n\to \infty} f_n= g \text{ in } \mathcal G. 
    \]
\begin{proof}
    Suppose that $\closure{\setB{f_n:n\in \N}}$ is compact and $\lim_{n\to \infty} E_{\mathcal O}(f_n) =0$.
    Let $f_{n_k}$ be a convergent subsequence of $f_n$ with limit $f \in \mathcal G$. 
    Since $E_{\mathcal O}(f_n) \to 0$, so does $E_{\mathcal O}(f_{n_k}) \to 0$. By continuity of $E_{\mathcal O}$ we also know that $E_{\mathcal O}(f_{n_k}) \to E_{\mathcal O}(f)$. By uniqueness of limits in $\R$ we have that $E_{\mathcal O}(f)=0$. By definition $f$ is a solution; but $O$ has a unique solution. Therefore, $f=g$ in $\mathcal G$. Since $\closure{\setB{f_n: n\in  \N}}$ is compact and every convergent subsequence of $\setB{f_n}_{n\in \N}$ converges to $g$, so must the sequence itself. 
    
    The converse is trivial, since the closure of the set consisting of the points of a convergent sequence in normed vector space (or metric space) is compact. And by continuity and definitions, we have $\lim_{n\to \infty} E_{\mathcal O}(f_n) =E_{\mathcal O}(g)=0$ whenever $\lim_{n\to \infty} f_n= g$.
\end{proof}

\textbf{Theorem \ref{thm:sound equtaion solve}} (Uniform-uniform convergence).
    Let $\mathcal G= C(X,Y)$ with the uniform norm, where $X$ is connected, compact and $\setB{a_n: n\in \N}$ is a sequence dense on $X$ ($\closure{\setB{a_n: n\in \N}}=X$).
    Suppose that $O:C(X,Y) \times X \to \R^d$ is a continuous function, we endow the space $\mathcal O =\setB{O(f,\cdot):f\in \mathcal G}$ with the uniform norm and let $\norm{\cdot}_{\R^d}$ be some fixed norm on $\R^d$.
    
    Assume that $O$ has a unique solution. Further, suppose that, $\mathcal F$ is the set of immediately representable functions of $\mathcal G$. Let $L: \mathcal F \to [0,\infty)$ be a function which satisfies
    \[
    \norm{f(x)-f(z)}_{Y} \leq  L(f)\norm{x-z}_X
    \]
    for all $x,z\in X$ and each $f\in \mathcal F$, where $\norm{\cdot}_X$ and $\norm{\cdot}_Y$ are some fixed norms on $X$ and $Y$ respectively. That is, $L(f)$ is a Lipschitz constant of $f\in \mathcal F$. Let $L_{\max}>0$ be some fixed number and let $x_0 \in X$, $y_0\in Y$ be some points.
    
    Then, 
    \begin{enumerate}
        \item Define the training error, 
        $E_{T,n}^{\infty}:\mathcal F \to \R$ by 
        \[
        E_{T,n}^{\infty}(f)=\max_{k=1,\dots, n}\setB{\norm{O(f,a_k)}_{\R^d}}+\max\setB{0,L(f)-L_{\max}}+\norm{f(x_0)-y_0}_Y
        \]
        for $n\in \N$. Then, $(\mathcal F,E_{T,n}^{\infty})$ soundly approximates the solution to $O$. i.e. Condition \eqref{def:sound appox} is satisfied. \label{thm:sound equtaion solve.max}
        \item  Define the training error, $E_{T,n}^{p}:\mathcal F \to \R$ for $p>0$ by 
        \[
        E_{T,n}^{p}(f)=\sum_{k=1}^n{\norm{O(f,a_k)}_{\R^d}^p}
        +\max\setB{0,L(f)-L_{\max}}+\norm{f(x_0)-y_0}_Y
        \]
        for $n\in \N$. Then, $(\mathcal F,E_{T,n}^{p})$ soundly approximates the solution to $O$. \label{thm:sound equtaion solve.psum} 
    \end{enumerate}

\begin{proof}
    We first prove \cref{thm:sound equtaion solve.max}, our proof strategy is to use \cref{thm:no training compact}. Suppose that $\setB{f_n}_{n\in \N}$ is a sequence of $\mathcal F$ with $E_{T,n}^{\infty}(f_n) \to 0$. For $r = \sup\setB{E_{T,n}^{\infty}(f_n) : n \in \N}$ we have that
    \[
    \setB{f\in \mathcal F: E_{T,n}^{\infty}(f) <r}\subseteq \setB{f\in \mathcal F: \max\setB{0,L(f)-L_{\max}}+\norm{f(x_0)-y_0}_Y < r} =:\mathcal F_{r}
    \]
    and $\setB{f_n :n\in \N}  \subseteq  \mathcal F_{r}$. Every function in $\mathcal F_r$ has Lipschitz constant less than $L_{\max} +r$, hence $\mathcal F_r$ is a uniformly equicontinuous set.
    By \cref{thm:localGraphConvergece} (in \cref{sec:unifrom norm}) we know that $\closure{\mathcal F_{r}}$ is compact in $C(X,Y)$. Therefore, $O:\closure{\mathcal F_r}\times X \to \R$ is uniformly continuous by the Heine-Cantor theorem.
    
    We now show that $E_{\mathcal O}(f_n)\to 0$. Pick $\epsilon >0$, by compactness of $X$ and continuity of $O$ there is a $x^*_n$ with  $\norm{O(f_n,x^*_n)}_{\R^d} =E_{\mathcal O}(f_n)=\sup\setB{\norm{O(f_n,x)}_{\R^d}: x\in X}$. 
    Since, $O$ is uniformly continuous on $\mathcal F_r\times X$ there is a $\delta >0$ such that for all $f \in \mathcal F_r$ and every $x,z\in X$ we have 
    \[
    \norm{x-z}_X <\delta \implies \norm{O(f,x)-O(f,z)}_{\R^d}  <\frac \epsilon 2.
    \]
    For this $\delta$, we can pick a $N_1\in \N$ so that for all $n\geq N_1$ we have $X\subseteq \union_{k=1}^n \B[X]{\delta}{a_k}$ ($\B[X]{\delta}{a_k}$ denotes the open ball of radius $\delta>0$ centered at $a_k$), since $\setB{a_n:n\in \N}$ is dense on $X$. Since, $E_{T,n}^{\infty}(f_n) \to 0$ there is a $N_2 \in \N$ with for all $n\geq N_2$ we have $E_{T,n}^{\infty}(f_n) < \frac \epsilon 2$.
    
    For all $n \geq \max\setB{N_1,N_2}$, let $a_{k^*_n}$ have $\norm{x^*_n-a_{k^*_n}}_X<\delta$ where $k^*_n=1,\dots, n$ and consider
    \begin{align*}
        \norm{O(f_n,x^*_n)}_{\R^d}=&E_{\mathcal O}(f_n) \\
        \norm{O(f_n,x^*_n)}_{\R^d}\leq & \norm{O(f_n,x^*_n)-O(f_n,a_{k^*_n})}_{\R^d} + \norm{O(f,a_{k^*_n})}_{\R^d} \\
        \norm{O(f_n,x^*_n)}_{\R^d}<& \frac \epsilon 2 +\max_{k=1,\dots, n}\setB{\norm{O(f,a_k)}_{\R^d}} \\
        E_{\mathcal O}(f_n) <& \frac \epsilon 2 + E_{T,n}^{\infty}(f_n) < \epsilon.
    \end{align*}
    Non-negativity of $	E_{\mathcal O}$ with the above means that 	$E_{\mathcal O}(f_n)\to 0$. It can be shown that $	E_{\mathcal O}$ is continuous, since $O$ is. By \cref{thm:no training compact} we have that $f_n$ converges to the unique solution of $O$. 
    
    The proof of \cref{thm:sound equtaion solve.psum} is nearly identical. The only divergence in the proofs happens when we bound $\norm{O(f,a_{k^*_n})}_{\R^d}$ by $ \sqrt[p]{\sum_{k=1}^n{\norm{O(f,a_k)}_{\R^d}^p}}$ and this by $\sqrt[p]{E_{T,n}^{p}(f_n)}$.
\end{proof}

\correction{%
\begin{remark}[On uniqueness of solutions to $O$]
    \Cref{thm:sound equtaion solve} assumes solutions to $\norm{O(g,\cdot)}_{\mathcal O}=0$ are unique. 
    This assumption cannot be weakened while maintaining sound approximation. 
    The reason for this is fundamental. \\
    Suppose that $g,h$ are distinct $O$-solutions and we choose $\mathcal F=\set{g,h}$. Since $\mathcal F$ is finite it is compact, so compactness terms are not necessary to add to the training error. Also suppose we have direct access to $\norm{O(f,\cdot)}_{\mathcal O}$, so we will pick $E_{T,n}(f)=\norm{O(f,\cdot)}_{\mathcal O}$. \\
    Now define a sequence $f_n$ by $f_{2k}=g$ and $f_{2k+1}=h$ for $k\in \N$. 
    We can see that $\lim_{n\to \infty}E_{T,n}(f_n)=0$ and but $\lim_{n\to \infty} f_n$ does not exist. \\
\end{remark}
}

\textbf{Theorem \ref{thm:MSE sound equtaion solve}} (Uniform-$L^2$ convergence).
    Let $\mathcal G= C(X,Y)$ with the uniform norm, where $X$ is connected, compact and $\setB{a_n: n\in \N}$ is a sequence dense on $X$ with the additional requirement that for all $n\in \N$
    \[
    \delta_n= \inf\setB{\delta >0: X\subseteq \union_{k=1}^n\B[X]{\delta}{a_k}}
    \]
    and
    \[
    C=\sup\setB{ n\mu(\B[X]{\delta_n}{a_1}) : n\in \N} <\infty
    \]
    where $\mu$ is the Lebesgue measure and $\B[X]{\delta}{x}$ is the open ball of radius $\delta >0$ centered at $x\in X$ with respect to some fixed norm $\norm{\cdot}_X$. 
    Suppose that $O:C(X,Y) \times X \to \R^d$ is a continuous function, we endow the space $\mathcal O =\setB{O(f,\cdot):f\in \mathcal G}$ with the $L^2$ norm. 
    
    Assume that $O$ has a unique solution. Further, suppose that, $\mathcal F$ is the set of immediately representable functions of $\mathcal G$. Let $L: \mathcal F \to [0,\infty)$ be a function which satisfies
    \[
    \norm{f(x)-f(z)}_{Y} \leq  L(f)\norm{x-z}_X
    \]
    for all $x,z\in X$ and each $f\in \mathcal F$, where $\norm{\cdot}_Y$ is some fixed norm on  $Y$. That is, $L(f)$ is a Lipschitz constant of $f\in \mathcal F$. Let $L_{\max}>0$ be some fixed number and let $x_0 \in X$, $y_0\in Y$ be some points.
    
    Then, the training error, $E_{T,n}^{MSE}:\mathcal F \to \R$ defined by
    \[
    E_{T,n}^{MSE}(f)=\frac{1}{n} \sum_{k=1}^n{\norm{O(f,a_k)}_{2}^2}
    +\max\setB{0,L(f)-L_{\max}}+\norm{f(x_0)-y_0}_Y
    \]
    for $n\in \N$ and $\norm{\cdot}_2$ is the $2$ norm on $\R^d$. Then, $(\mathcal F,E_{T,n}^{MSE})$ soundly approximates the solution to $O$, i.e. Condition \ref{def:sound appox} is satisfied.
\begin{proof}
    Again, our proof strategy is to use \cref{thm:no training compact}. Suppose that $\setB{f_n}_{n\in \N}$ is a sequence of $\mathcal F$ with $E_{T,n}^{MSE}(f_n) \to 0$. For $r = \sup\setB{E_{T,n}^{MSE}(f_n) : n \in \N}$ then we have
    \[
    \setB{f\in \mathcal F: E_{T,n}^{MSE}(f) <r}\subseteq \setB{f\in \mathcal F: \max\setB{0,L(f)-L_{\max}}+\norm{f(x_0)-y_0}_Y < r} =:\mathcal F_{r}
    \]
    and $\setB{f_n :n\in \N}  \subseteq  \mathcal F_{r}$. Every function in $\mathcal F_r$ has Lipschitz constant less than $L_{\max} +r$, hence $\mathcal F_r$ is a uniformly equicontinuous set.
    By \cref{thm:localGraphConvergece} we now that $\closure{\mathcal F_{r}}$ is compact in $C(X,Y)$. Therefore, $O:\closure{\mathcal F_r}\times X \to \R$ is uniformly continuous by the Heine-Cantor theorem.
    
    We now show that $E_{\mathcal O}(f_n)\to 0$. 
    Pick $\epsilon >0$. 
    Since, $O$ is uniformly continuous on $\mathcal F_r\times X$ there is a $\delta >0$ such that for all $f \in \mathcal F_r$ and every $x,z\in X$ we have 
    \[
    \norm{x-z}_X <\delta \implies \norm{O(f,x)-O(f,z)}_{p}  <\min\setB{\sqrt{\frac{\epsilon}3}, \frac \epsilon {6M_{\mathcal F_r}}}
    \]
    where $M_{\mathcal F_r}$ is a uniform upper bound on $\norm{O(f,x)}_2$ for $f\in \closure{\mathcal F_r}$ and $x\in X$ (this bound exists by compactness of  $\closure{\mathcal F_r}\times X$ and continuity of $O$).
    
    For this $\delta$, we can pick a $N_1\in \N$ so that $\delta_{N_1}<\delta$ it follows that for all $n\geq N_1$ we have $X\subseteq \union_{k=1}^n \closure{\B[X]{\delta_n}{a_k}}$ and $\delta_n <\delta$.  Since, $E_{T,n}^{MSE}(f_n) \to 0$ there is a $N_2 \in \N$ with for all $n\geq N_2$ we have $E_{T,n}^{MSE}(f_n) < \frac \epsilon {3}$.

    Then, for all $n\geq \max\setB{N_1,N_2}$
    \begin{align*}
        {E_{\mathcal O}(f_n)}^2&=\int_{x\in X}\norm{O(f_n,x)}_2^2dx \\&\leq \int_{x\in \union_{k=1}^n\B[X]{\delta_n}{a_k}}\norm{O(f_n,x)}_2^2dx \\
        &\leq \sum_{k=1}^n \int_{x\in\B[X]{\delta_n}{a_k}}\norm{O(f_n,x)}_2^2dx \\
        &\leq \sum_{k=1}^n \norm{O(f_n,x_k^*)}_2^2\mu\round{\B[X]{\delta_n}{a_k}}
    \end{align*}
    where $x_k^* \in \B[\R^d]{\delta_n}{a_k}$ is the maximizer of  $\norm{O(f_n,x)}$ in $\closure{\B[X]{\delta_n}{a_k}}$. Let $\mu\round{\B[\R^d]{\delta_n}{a_k}} =\mu_n$; as the measure is translation invariant $\mu_k$ does not depend on $k$. 
    Then, by uniform continuity of $O$ we see
    \begin{align*}
        &\sum_{k=1}^n \norm{O(f_n,x_k^*)}_2^2\mu_n\\
        &\leq \mu_n \sum_{k=1}^n \round{\norm{O(f_n,x_k^*)-O(f_n,a_k)}_2 +\norm{O(f_n,a_k)}_2}^2 \\
        &< \mu_n \sum_{k=1}^n\round{ \min\setB{\sqrt{\frac{\epsilon}3}, \frac \epsilon {6M_{\mathcal F_r}}} + \norm{O(f_n,a_k)}_2}^2 \\
        &= \mu_n \sum_{k=1}^n \round{\min\setB{\sqrt{\frac{\epsilon}3}, \frac \epsilon {6M_{\mathcal F_r}}}^2 +2\min\setB{\sqrt{\frac{\epsilon}3}, \frac \epsilon {6M_{\mathcal F_r}}}\norm{O(f_n,a_k)}_2 +\norm{O(f_n,a_k)}_2^2} \\
        &\leq \mu_n \sum_{k=1}^n \round{\frac \epsilon 3 + \frac \epsilon 3 \frac {\norm{O(f_n,a_k)}} {M_{\mathcal F_r}} +\norm{O(f_n,a_k)}_2^2} \\
        &\leq \mu_n \sum_{k=1}^n \round{\frac \epsilon {3}+\frac \epsilon 3} +\mu_n \sum_{k=1}^n  \norm{O(f_n,a_k)}_2^2 \\
        &= \epsilon \frac {2n\mu_n} {3} +\mu_n n E_{T,n}^{MSE}(f_n)\\
        &< \epsilon {n\mu_n} \leq \epsilon C.
    \end{align*}
    Where $C$ is a constant given from the statement of this theorem. This shows that $E_{\mathcal O}(f_n) \to 0$. Again, it is possible to show that $E_{\mathcal O}$ is continuous. Therefore, by \cref{thm:no training compact} the sequence $f_n$ converges to unique solution of $O$, as required.
\end{proof}

\section{Compactness terms}\label{sec:compactness_terms}

In this section we further discus terms which can be added to traditional loss functions, which cause the approximating function sequence, $\setB{f_n}_{n\in \N}$, to have compact closure. 
\begin{definition}\label{def:compactness_terms}
    A function $\kappa_{\mathcal G}: \mathcal G \to [0,\infty)\cup \setB{\infty}$ is called a $\mathcal G$ compactness term if: 
    The sublevel sets of $\kappa_{\mathcal G}$ are nonempty and have compact closure. That is for all $r \in [0,\infty)$ the set $\closure{\setB{g\in \mathcal G: \kappa_{\mathcal G}(g) < r}}$ is nonempty and compact in $\mathcal G$.

    The function $\kappa_{\mathcal G}$ is called an eventual $\mathcal G$ compactness term if: There is a $R\in [0,\infty)$ such that $r \in [0,R]$ the set $\closure{\setB{g\in \mathcal G: \kappa_{\mathcal G}(g) < r}}$ is nonempty and compact in $\mathcal G$. 

    We say a function $\gamma:\mathcal F \to [0,\infty)\cup \setB{\infty}$, where $\mathcal F\subseteq \mathcal G$, is called a (eventual) $\mathcal G$ compactness term on extension if: The function $\hat \gamma: \mathcal G \to [0,\infty)\cup \setB{\infty}$, defined by $\hat \gamma(g)=\gamma(g)$ for $g\in \mathcal F$ and $\hat \gamma(g)=\infty$ otherwise, is a (eventual) $\mathcal G$ compactness term.
\end{definition}

It is not hard to see that a $\mathcal G$ compactness term is an eventual $\mathcal G$ compactness term. One can also see the following proposition holds.
\begin{proposition}
    Let $\setB{f_n}_{n\in \N}$ be a sequence of functions in $\mathcal G$ and $\kappa_{\mathcal G}$ be an eventual $\mathcal G$ compactness term. 
    If $\lim_{n\to \infty} \kappa_{\mathcal G}(f_n)=0$ then, $\closure{\setB{f_n:m\in \N}}$ is compact in $\mathcal G$. 
    Moreover, any $g\in \mathcal G$ which is the limit of a convergent subsequence of $\setB{f_n}_{n\in \N}$, has $g\in \bigcap_{r>0}\closure{\setB{h\in \mathcal G: \kappa_{\mathcal G}(h)\leq r}}$.
\end{proposition}

We present some examples of compactness terms.  
\begin{proposition}[Some uniform compactness terms]\label{thm:some_compact_terms}
    Let $X$ be a connected compact set. Let $\mathcal G=C(X,Y)$ endowed with the uniform norm. Suppose that $x_0\in X$ and $y_0\in Y$ are some points. 
    The following functions are $\mathcal G$ compactness terms on extension:
    \begin{enumerate}
        \item  $\displaystyle \gamma_{L}(f)= \max\set{0,L(f)-L_{\max}} + \norm{f(x_0)-y_0}_Y$. 
        Where $L_{\max}>0$ is a constant and $\mathcal F$ is a set of Lipchitz functions, with a function $L:\mathcal F\to [0,\infty)$ such that $L(f)$ is a upper bound on the Lipchitz constant of $f\in \mathcal F$. \label{thm:some_compact_terms.Lipchitz}
        \item $\displaystyle \gamma_{\alpha,L}(f)= \max\set{0,L_{\max}-L(f)} +\max\set{0,\alpha_{\min}-\alpha(f)}+ \norm{f(x_0)-y_0}_Y$. 
        Where  $L_{\max},\alpha_{\min} >0$ are constants, $\mathcal F$ is a set of $\alpha$ Holder continuous functions, with functions $L:\mathcal F\to [0,\infty)$ and $\alpha:\mathcal F\to [0,1]$ such that
        \[
          \norm{f(x)-f(z)}_{Y} \leq  L(f)\norm{x-z}_X^{\alpha(f)}
    \]
    for all $x,z\in X$ and each $f\in \mathcal F$. \label{thm:some_compact_terms.Holder}
        \item  $\displaystyle \gamma_{\omega}(f)= \max\set*{0,\limsup_{\delta\to 0^+}\ln\round{\frac{\omega_f(\delta)}{\omega(\delta)}}}  + \norm{f(x_0)-y_0}_Y$. 
        Where, for  $f\in \mathcal F$, $\omega,\omega_f:[0,\infty) \to [0,\infty)\cup \set{\infty}$ are functions \correction{that} are zero at zero and continuous at zero. Moreover for each $f\in \mathcal F$ we have
        \[
            \norm{f(x)-f(z)}_{Y} \leq  \omega_f(\norm{x-z}_X).
        \] \label{thm:some_compact_terms.impractical}
    \end{enumerate} 
\end{proposition}
\begin{proof}[Sketch of proof]
    \Cref{thm:some_compact_terms.impractical} follows from the fact that if $\hat \gamma_{\omega}(g)\leq r$, $r>0$ then, $g\in \mathcal F$ and there is a $\Delta >0$ such that for all $\delta \in (0,\Delta]$ we have that 
    \[
        \omega_g(\delta) \leq e^{2r} \omega(\delta).
    \]
    Equicontinuity of $\set{g\in \mathcal G: \gamma_{\omega}(g)\leq r}$ follows from continuity of $\omega$ at zero. Thus the set $\set{g\in \mathcal G: \gamma_{\omega}(g)\leq r}$ is equicontinuous and pointwise bounded and so by \cref{thm:localGraphConvergece} we have that is closure is compact. 
\end{proof}

Inspection of the proofs of \cref{thm:sound equtaion solve} and \cref{thm:MSE sound equtaion solve} reveals that the compactness term of \cref{thm:some_compact_terms.Lipchitz} of \cref{thm:some_compact_terms} can be replaced with any other compactness term in \cref{thm:some_compact_terms}. 

\cref{thm:some_compact_terms.Lipchitz} of \cref{thm:some_compact_terms} gives a theoretically widely applicable compactness term; it is defined on all of $C(X,Y)$. 
However, it is very unclear how one can ever compute this compactness term. 
Therefore, if one really wants to use a compactness term in practice then, they should use another actually \correction{computable} compactness term. 
Note that it might be more appropriate to design other compactness terms that are more specific to the underlying neural network architecture, rather than the broadly defined ones given in \cref{thm:some_compact_terms}.

One may wonder about the topological assumptions on the set $X$. Why must $X$ be connected and compact? The compactness assumption is necessary to work in a normed space. 
The connectedness assumption allows us to use \cref{thm:localGraphConvergece}. 
If the space $X$ is only assumed to be compact then, there is a $N\in \N$ with $X=\union_{n=0}^N X_n$, where each $X_n$ is connected and compact. 
Further, if $(x_n,y_n) \in X_n\times Y$ we can replace the data point terms $\norm{f(x_0)-y_0}_Y$ of \cref{thm:some_compact_terms} with $\sum_{n=0}^N\norm{f(x_n)-y_n}_Y$ then, these modified $\gamma$'s will be compactness terms for the disconnected compact space $X$.



\section{Proof of Random Collocation Convergence}\label{sec:app:test_error}

We write \(Y_n \xrightarrow{p} Y\) (or \(\operatorname*{plim}_{n\to\infty} Y_n = Y\)) to denote convergence in probability, i.e.,
\[
\forall \varepsilon>0:\quad
\Pb(|Y_n-Y|>\varepsilon)\to 0 \quad \text{as } n\to\infty.
\]

Let $(X,\mathcal B,\mu)$ be a probability space and let $(x_k)_{k\in\mathbb N}$ be i.i.d.\ $X$-valued random variables with distribution $\mu$ defined on a probability space $(\Omega,\mathcal F,\Pb)$.

Let $O:\mathcal G\times X\to\mathbb R^d$ and for $p\geq 1$ define
\[
E_{\mathcal O}(f)
=\left(\int_X \|O(f,x)\|_{\mathbb R^d}^p d\mu(x)\right)^{1/p}.
\]
For $n\in\mathbb N$ define the empirical residual
\[
E_{\mathrm{test},p,n}(f,\omega)
=\frac{1}{n}\sum_{k=1}^n \|O(f,x_k(\omega))\|_{\mathbb R^d}^p .
\]

\begin{theorem}[Uniform law of large numbers on a compact hypothesis set]\label{thm:ulln_compact}
Let $X$ be compact and let $\mu$ be a probability measure on $X$. Let $x_1,x_2,\dots \stackrel{i.i.d.}{\sim}\mu$. Fix $p\ge 1$ and let $\mathcal G$ be a metric space. Let $\mathcal F\subset \mathcal G$ be compact.

Assume $O:\mathcal F\times X\to\R^d$ is continuous. Then
\[
\sup_{f\in\mathcal F}\left|
E_{\mathrm{test},p,n}(f,\omega)-E_{\mathcal O}(f)^p
\right|\xrightarrow{a.s.}0.
\]
Equivalently, $\{h_f: f\in\mathcal F\}$ is $\mu$-Glivenko-Cantelli \citep{talagrand1987glivenko}.
\end{theorem}

\begin{proof}
Define $h_f(x)=\|O(f,x)\|^p$.
Since $\mathcal F\times X$ is compact and $O$ is continuous, $h_f(x)$ is continuous on $\mathcal F\times X$ and hence bounded: there exists $M<\infty$ such that $0\le h_f(x)\le M$ for all $(f,x)\in\mathcal F\times X$.

Moreover, continuity on the compact set implies uniform continuity by the Heine-Cantor theorem \citep[Proposition 5.8.2]{sutherland2009introduction}: for every $\varepsilon>0$ there exists $\delta>0$ such that
\[
\|f-g\|_\Gc<\delta \implies \|h_f(x)-h_g(x)\|_\infty<\varepsilon.
\]
By compactness of $\mathcal F$, for all $\delta>0$, there exists a set $\{f^1,\dots,f^N\}\subset\mathcal F$ so that for every $f\in\mathcal F$ there is an index $i$ such that $\|f-f^i\|_\Gc<\delta$.
Fix $f$ and $f^i$. Then for every sample sequence,
\begin{align*}
    |E_{\mathrm{test},p,n}(f,\omega)-E_O(f)|
    &\leq|E_{\mathrm{test},p,n}(f,\omega)-E_{\mathrm{test},p,n}(f^i,\omega)|+|E_{\mathrm{test},p,n}(f^i,\omega)-E_O^p(f^i)|+|E_O^p(f^i)-E_O^p(f)|\\
    &=|\frac1n\sum_{k=1}^{n}(h_f(x_k)-h_{f^i}(x_k))|+|E_{\mathrm{test},p,n}(f^i,\omega)-E_O(f^i)|+|(h_{f^i}-h_{f})d\mu|\\
    &\leq 2\varepsilon+\max_{1\le i\le N}|E_{\mathrm{test},p,n}(f^i,\omega)-E_O(f^i)|.
\end{align*}
Since each $h_{f^i}(x_k)$ is i.i.d.\ bounded in $[0,M]$, Hoeffding's inequality \citep{serfling1974probability} gives
\[
\Pr\left(|E_{\mathrm{test},p,n}(f^i,\omega)-E_O(f^i)|>t\right)
\le 2\exp\left(-\frac{2n t^2}{M^2}\right).
\]
By a union bound over $i=1,\dots,N$,
\[
\Pr\left(\max_{1\le i\le N}|E_{\mathrm{test},p,n}(f^i,\omega)-E_O(f^i)|>t\right)
\le 2N\exp\left(-\frac{2n t^2}{M^2}\right).
\]
The right-hand side is summable in $n$, so by Borel-Cantelli \citep{durrett2019probability},
\[
\max_{1\le i\le N}|E_{\mathrm{test},p,n}(f^i,\omega)-E_O(f^i)|\xrightarrow{a.s.}0.
\]
Since $\varepsilon>0$ was arbitrary, we arrive at our desired conclusion.
\[
\sup_{f\in\mathcal F}\left|
E_{\mathrm{test},p,n}(f,\omega)-E_{\mathcal O}(f)^p
\right|\xrightarrow{a.s.}0.
\]
\end{proof}

\begin{theorem}[Random collocation convergence via ULLN]\label{thm:rand_coll_ulln}
Assume the hypotheses of Theorem~\ref{thm:ulln_compact}.
Let $(f_n)_{n\ge 1}$ be any (possibly data-dependent) sequence with $f_n\in\mathcal F$.
If
\[
E_{\mathrm{test},p,n}(f_n)=\frac1n\sum_{k=1}^n \|O(f_n,x_k)\|^p \xrightarrow{p} 0,
\]
then
\[
E_{\mathcal O,p}(f_n)^p=\int_X \|O(f_n,x)\|^p d\mu(x)\longrightarrow 0.
\]
Consequently, if $O$ has a unique solution $g\in\mathcal G$ characterized by
$E_{\mathcal O,p}(f)=0\Rightarrow f=g$, then $f_n\to g$ in $\mathcal G$.
\end{theorem}

\begin{proof}
By Theorem~\ref{thm:ulln_compact},
\[
\sup_{f\in\mathcal F}\left|E_{\mathrm{test},p,n}(f)-E_{\mathcal O,p}(f)^p\right|\xrightarrow{a.s.}0,
\]
hence also in probability. Therefore
\[
E_{\mathcal O,p}(f_n)^p
\le E_{\mathrm{test},p,n}(f_n) + \sup_{f\in\mathcal F}\left|E_{\mathrm{test},p,n}(f)-E_{\mathcal O,p}(f)^p\right|
\xrightarrow{p} 0,
\]
so $E_{\mathcal O,p}(f_n)\to 0$.
If the solution is unique in the sense stated, every convergent subsequence of $(f_n)$
must converge to $g$, and compactness of $\mathcal F$ forces $f_n\to g$.
\end{proof}

\section{Case studies and a posteriori error estimates}\label{app:case_studies}

In this section, we present a few case studies to illustrate the applicability of the theoretical analysis presented in the previous sections. We present initial value problems for ODEs and consider two typical classes of PDEs, namely elliptic and parabolic equations. We show that the assumptions of the theoretical results in the previous sections can be easily verified. There are two main assumptions: 1) \(O:C(X,Y) \times X \to \mathbb{R}^d\) is a continuous function; 2) \(O\) has a unique solution. We shall verify these for each case study.

Furthermore, the convergence results from the previous section are of theoretical interest because they characterize what occurs in the limit of the training process. In practice, the training process is finite; this includes the number of points used in \(E_{T,n}\) and the number of steps taken by optimization algorithms before converging to the true minimizers. To this end, for each case study, we provide an a posteriori error analysis for neural network approximations of solutions to differential equations. More specifically, let \(\phi\) be a given neural network approximation of the solution \(\varphi\) to equation \(O\). The main objective of a posteriori error analysis is to compute upper bounds for the generalization error \(E_g(\phi) = \norm{\phi-\varphi}_{\mathcal{G}}\), based on estimates of the equation error \(E_{\mathcal{O}}(\phi)\). Note that, unlike in convergence analysis, obtaining sharper estimates might require breaking down \(E_{\mathcal{O}}(\phi)\) into components, such as initial error, boundary error, and ODE/PDE residual errors.

\begin{lemma}\label{lem:harmonic_extension}
Let $\Omega\subset\R^n$ be a bounded convex Lipschitz domain.
Fix $T>0$. Assume the boundary condition
$\varepsilon_4:\partial\Omega\times[0,T]\to \mathbb{R}$
satisfies
$\varepsilon_4 \in C^2\big([0,T]; H^{1/2}(\partial\Omega)\big),$
i.e. the map $t\to \varepsilon_4(\cdot,t)$ is twice continuously differentiable as a function with values in $H^{1/2}(\partial\Omega)$.

For each $t\in[0,T]$, let $r(\cdot,t)\in H^1(\Omega)$ be the harmonic lifting of $\varepsilon_4(\cdot,t)$, i.e. $r(\cdot,t)$ is the unique solution of
\begin{equation}\label{eq:harmonic_lift}
\Delta r(\cdot,t)=0 \ \text{in }\Omega, \qquad 
r(\cdot,t)=\varepsilon_4(\cdot,t)\ \text{on }\partial\Omega.
\end{equation}
Then the map $t\to r(\cdot,t)$ belongs to $C^2\big([0,T]; H^1(\Omega)\big)$.
In particular, $r_t(\cdot,t)$ and $r_{tt}(\cdot,t)$ exist in $H^1(\Omega)$ for every $t\in[0,T]$ and depend continuously on $t$ in the $H^1(\Omega)$ norm.
\end{lemma}

\begin{proof}
For each $g\in H^{1/2}(\partial\Omega)$, denote by $\xi(g)\in H^1(\Omega)$ the unique weak solution of
\[
\Delta (\xi(g))=0 \ \text{in }\Omega,\qquad \xi(g)|_{\partial\Omega}=g,
\]
i.e. the harmonic extension of $g$ to $\Omega$. The following are true by \cite{gilbarg2001elliptic}:
\begin{itemize}
\item (\emph{Linearity}) $\xi$ is linear.
\item (\emph{Boundedness}) There exists a constant $C_\Omega>0$ depending on the domain $\Omega$ \citep{acosta2004optimal,payne1960optimal} such that
\begin{equation}\label{eq:P_bound}
\|\xi(g)\|_{H^1(\Omega)} \le C_\Omega \|g\|_{H^{1/2}(\partial\Omega)}\qquad \forall g\in H^{1/2}(\partial\Omega).
\end{equation}
\item (\emph{Uniqueness}) $\xi(g)$ is uniquely determined by $g$.
\end{itemize}
Thus $\xi:H^{1/2}(\partial\Omega)\to H^1(\Omega)$ is a bounded linear operator.
By definition of the harmonic lift \eqref{eq:harmonic_lift}, for each $t\in[0,T]$ we have
\[
r(\cdot,t) = \xi\big(\varepsilon_4(\cdot,t)\big).
\]
Hence $t\to r(\cdot,t)$ is the composition of the map $t\to \varepsilon_4(\cdot,t)$ with the bounded linear operator $\xi$.

Next, we fix $t\in[0,T]$ and consider the derivative in the sense of $H^1(\Omega)$:
\[
\frac{r(\cdot,t+h)-r(\cdot,t)}{h}
=
\xi\left(\frac{\varepsilon_4(\cdot,t+h)-\varepsilon_4(\cdot,t)}{h}\right).
\]
Using \eqref{eq:P_bound} and the fact that $\varepsilon_4\in C^1([0,T];H^{1/2}(\partial\Omega))$, the difference quotient
\[
\frac{\varepsilon_4(\cdot,t+h)-\varepsilon_4(\cdot,t)}{h}
\to \partial_t \varepsilon_4(\cdot,t)
\quad \text{in } H^{1/2}(\partial\Omega) \quad \text{as } h\to 0.
\]
Applying the bounded linear operator $\xi$ yields
\[
\frac{r(\cdot,t+h)-r(\cdot,t)}{h}
\to
\xi\big(\partial_t \varepsilon_4(\cdot,t)\big)
\quad \text{in } H^1(\Omega).
\]
Therefore $r$ is differentiable in time with
\[
\partial_t r(\cdot,t) = \xi\big(\partial_t \varepsilon_4(\cdot,t)\big)\in H^1(\Omega).
\]
Moreover, since $t\to \partial_t \varepsilon_4(\cdot,t)$ is continuous in $H^{1/2}(\partial\Omega)$ and $\xi$ is bounded, it follows that $t\to r_t(\cdot,t)$ is continuous in $H^1(\Omega)$; hence
\[
r \in C^1\big([0,T];H^1(\Omega)\big).
\]

Repeating the same argument with $r_t$ in place of $r$, we write
\[
\frac{r_t(\cdot,t+h)-r_t(\cdot,t)}{h}
=
\xi\left(\frac{\partial_t\varepsilon_4(\cdot,t+h)-\partial_t\varepsilon_4(\cdot,t)}{h}\right).
\]
Since $\varepsilon_4\in C^2([0,T];H^{1/2}(\partial\Omega))$, we have
\[
\frac{\partial_t\varepsilon_4(\cdot,t+h)-\partial_t\varepsilon_4(\cdot,t)}{h}
\to \partial_{tt}\varepsilon_4(\cdot,t)
\quad \text{in } H^{1/2}(\partial\Omega),
\]
and applying $\xi$ gives convergence in $H^1(\Omega)$:
\[
\frac{r_t(\cdot,t+h)-r_t(\cdot,t)}{h}
\to \xi\big(\partial_{tt}\varepsilon_4(\cdot,t)\big).
\]
Thus $r_t$ is time-differentiable and
\[
\partial_{tt} r(\cdot,t) = \xi\big(\partial_{tt}\varepsilon_4(\cdot,t)\big)\in H^1(\Omega).
\]
Continuity of $t\to w_{tt}(\cdot,t)$ in $H^1(\Omega)$ again follows from continuity of
$t\to \partial_{tt}\varepsilon_4(\cdot,t)$ in $H^{1/2}(\partial\Omega)$ and boundedness of $P$.
Therefore $r\in C^2([0,T];H^1(\Omega))$.
\end{proof}

\subsection{Proof of Theorem~\ref{thm:generalization}}\label{sec:app:thm4_proof}

We first record two standard facts.

\begin{lemma}[Modulus of continuity for compact families]\label{lem:modulus_compact}
Let $X$ be compact and let $\Fc\subset C(X,Y)$ be compact under $\|\cdot\|_\infty$. Then the function
\[
\omega_{\Fc}(r):=\sup\{\|f(x)-f(y)\|_Y: f\in\Fc,\, x,y\in X,\, \|x-y\|_X\le r\}
\]
satisfies $\omega_{\Fc}(r)<\infty$ for all $r$ and $\omega_{\Fc}(r)\downarrow 0$ as $r\downarrow 0$.
\end{lemma}

\begin{proof}
Since $\Fc$ is compact in $(C(X,Y),\|\cdot\|_\infty)$, it is uniformly equicontinuous by \cref{thm:AA}, hence for every $\varepsilon>0$ there exists $\delta>0$ such that $\|x-y\|_X\le \delta$ implies $\sup_{f\in\Fc}\|f(x)-f(y)\|_Y\le \varepsilon$.
This is exactly $\omega_{\Fc}(\delta)\le \varepsilon$, proving $\omega_{\Fc}(r)\to 0$ as $r\to 0$.
Finiteness for each $r$ follows from boundedness of $\Fc$ and compactness of $X$.
\end{proof}

\begin{lemma}[Fill distance implies pointwise approximation]\label{lem:fill_distance}
Let $h\in C(X,Y)$ and let $\{a_k\}_{k=1}^n\subset X$ have fill distance
$\delta_n=\sup_{x\in X}\min_k\|x-a_k\|_X$.
Then
\[
\sup_{x\in X}\|h(x)\|_Y\le \max_{1\le k\le n}\|h(a_k)\|_X + \omega_{\{h\}}(\delta_n),
\]
where $\omega_{\{h\}}(r):=\sup\{\|h(x)-h(y)\|_Y:\ \|x-y\|_X\le r\}$ is the modulus of continuity of $h$.
\end{lemma}

\begin{proof}
Fix $x\in X$ and choose $k(x)\in\arg\min_{1\le k\le n}\|x-a_k\|$, so that $\|x-a_{k(x)}\|\le \delta_n$.
Then
\[
\|h(x)\|
\le \|h(a_{k(x)})\| + \|h(x)-h(a_{k(x)})\|
\le \max_{1\le k\le n}\|h(a_k)\| + \omega_{\{h\}}(\delta_n).
\]
Taking the supremum over $x\in X$ yields the claim.
\end{proof}

\begin{proof}[Proof of Theorem~\ref{thm:generalization}]
Let $f\in\Fc$ be arbitrary and define the residual function $h_f(x):=O(f,x)$. By the assumption
\[
\|f-g\|_\infty \le C_0\|h_f\|_{\mathcal O}.
\]
If $\mathcal O=C(X,Y)$ with $\|\cdot\|_{\mathcal O}=\|\cdot\|_\infty$, then we can bound $\|h_f\|_\infty$ using the sampling set $\{a_k\}$. Indeed, Lemma~\ref{lem:fill_distance} gives
\[
\|h_f\|_\infty \le \max_{1\le k\le n}\|h_f(a_k)\|_Y + \omega_{\{h_f\}}(\delta_n).
\]
Since $f$ ranges in the compact set $\Fc$ and $O$ is continuous on $\Fc\times X$,
the residual family $O(\Fc)=\{h_f:\ f\in\Fc\}$ is compact in $C(X,Y)$
(as a continuous image of a compact set), hence uniformly equicontinuous.
By Lemma~\ref{lem:modulus_compact}, there exists a common modulus of continuity
\[
\omega_{\Fc}(r):=\sup_{h\in O(\Fc)}\sup_{\|x-y\|\le r}\|h(x)-h(y)\|
\quad\text{with}\quad
\omega_{\Fc}(r)\to 0\ \text{as } r\to 0,
\]
and in particular $\omega_{\{h_f\}}(r)\le \omega_{\Fc}(r)$ for each $f\in\Fc$.
Therefore,
\[
\|h_f\|_\infty
\le \max_{1\le k\le n}\|O(f,a_k)\| + \omega_{\Fc}(\delta_n).
\]
Combining with stability yields
\[
\|f-g\|_\infty
\le C_0\left(\max_{1\le k\le n}\|O(f,a_k)\| + \omega_{\Fc}(\delta_n)\right),
\]
which is exactly \cref{eq:generalization_representative} in the uniform-norm residual case.

\end{proof}

\subsection{Initial value problems for ODEs}
We first consider the IVP for ODEs. 

\begin{proposition}\label{prop:ivp}Consider the IVP
    \[
    \dot \phi(t) = F\circ \phi(t), \quad \phi(0)=\phi_0,
    \]
    on $[0,T]$, where $T>0$. 
        Let $\norm{\cdot}_{\R^{2d}}$ be the euclidean norm. Endow the space $C^1([0,T],\R^d)$ with the norm $\norm{g}_{C^1}=\max\setB{\norm{g}_{\infty},\norm{\dot g}_{\infty}}$ where, 
        $\norm{g}_{\infty}=\sup_{t\in [0,T]}\norm{g(t)}_{\R^{d}}$.
        Suppose that $F:\, \R^d\rightarrow\R^d$ is Lipschitz continuous. Define $O$ to be
    \[
    O(\phi,t)= 
    \begin{pmatrix}
        \dot \phi(t) -F\circ \phi(t) \\
        \phi(0)-\phi_0
    \end{pmatrix}.
    \]
Then $O: C^1([0,T],\R^d) \times [0,T] \to \R^{2d}$ is continuous. Furthermore, let $\phi$ be an approximate solution and $\varphi$ be any true solution. Then 
        \begin{equation}
        \begin{aligned}
    E_{\varphi}(\phi)= \norm{\phi - \varphi}_{\infty} & \leq E^I_{\mathcal{O}}(\phi)e^{LT} + E^R_{\mathcal{O}}(\phi)\frac{e^{LT}-1}{L} \\
         &\le E_{\mathcal{O}}(\phi) \left(e^{LT} + \frac{e^{LT}-1}{L} \right),
    \end{aligned}\label{eq:ivp_est}
        \end{equation}
    where $E^I_{\mathcal{O}}(\phi)=\norm{\phi(0)-\phi_0}$ denotes the initial error, $E^R_{\mathcal{O}}(\phi)=\sup_{t\in [0,T]}\norm{\dot \phi(t) -F\circ \phi(t)}$ is the ODE residual error, and $E_{\mathcal O}(\phi)=\norm{O(\phi,\cdot)}_{\infty} = \sup\setB{\norm{O(\phi,t)}_{\R^{2d}} :t\in [0,T]}$, and $L$ is the Lipschitz constant of $F$. Clearly, we have $\max(E^I_{\mathcal{O}}(\phi), E^R_{\mathcal{O}}(\phi))\le E_{\mathcal O}(\phi)$. In particular, the estimate (\ref{eq:ivp_est}) shows that \(O\) has a unique solution. 
\end{proposition}




\begin{proof}
    For brevity, let 
    $$
    \varepsilon \coloneqq 
    E^R_{\mathcal{O}}(\phi)=\sup_{t\in [0,T]}\norm{\dot \phi(t) -F\circ \phi(t)}. 
    $$
    We have
    \begin{equation}
        \varphi(t) = \phi_0 + \int_0^t F\circ \varphi(s)ds,\quad t\in[0,
        T], 
    \end{equation}
    and
    \begin{align*}
        \phi(t) & = \phi(0) + \int_0^t \dot{\phi}(s)ds = \int_0^t [F\circ \phi(s) + (\dot{\phi}(s) - F\circ \phi(s))] ds, \quad t\in[0,
        T]. 
    \end{align*}
    Combining the equations above leads to 
    \begin{align*}
        \norm{\phi(t) - \varphi(t)} &\le  \norm{\phi(0) - \phi_0} + \int_0^t 
        \left(L \norm{\phi(s) - \varphi(s)} + \varepsilon \right)ds \\
        &  = \norm{\phi(0) - \phi_0} + \varepsilon t + L \int_0^t 
        \norm{\phi(s) - \varphi(s)} ds. 
    \end{align*}
    By Gronwall's inequality, with $\alpha(t)=\norm{\phi(0) - \phi_0} + \varepsilon t$ and $\beta(t)= L$, and integration by parts, we obtain  
    \begin{align*}
        \norm{\phi(t) - \varphi(t)} &\le   \alpha(t) + \int_0^t L\alpha(s) e^{L(t-s)}ds \\
        &  = \alpha(t) - \alpha(t) e^{L(t-s)} \vert^{s=t}_{s=0} + \varepsilon \int_0^t e^{L(t-s)}ds \\
        & = \norm{\phi(0) - \phi_0} e^{L t}  -  \frac{\varepsilon}{L} e^{L(t-s)} \vert^{s=t}_{s=0} \\
        & = \norm{\phi(0) - \phi_0} e^{L t} +  \frac{\varepsilon}{L}(e^{Lt} - 1),
    \end{align*}
    for all $t\in [0,T]$, which proves (\ref{eq:ivp_est}). Clearly, (\ref{eq:ivp_est}) also implies that \(O\) has a unique solution. 
\end{proof}

\begin{remark}\label{rem:estimate}
After training a neural network approximate solution, we can empirically estimate the true solution error using (\ref{eq:ivp_est}) as follows. We can simply evaluate $E^I_{\mathcal{O}}(\phi)=\norm{\phi(0)-\phi_0}$. Furthermore, we can estimate $E^R_{\mathcal{O}}(\phi)$ by 
    \begin{equation}
    E^R_{\mathcal{O}}(\phi)\le \max_{n=1,\dots,N}\norm{\dot \phi(t_n) -F\circ \phi(t_n)}+L_{R}\Delta t,\label{eq:inv_est_test}		    
    \end{equation}
where $L_{R}$ is a Lipschitz constant of $\dot \phi(\cdot) -F\circ \phi(\cdot)$ on $[0,T]$, and $\Delta t$ is the largest difference between the adjacent $t_n$'s, with $\{t_n\}$ being the test points at which the ODE residual $\dot \phi(t_n) -F\circ \phi(t_n)$ is evaluated.  Inequality (\ref{eq:inv_est_test}) follows from the triangle inequality and the Lipschitz continuity of $\dot \phi(\cdot) -F\circ \phi(\cdot)$ on $[0,T]$. Note that the right-hand side of (\ref{eq:inv_est_test}) represents an upper bound on the true ODE residual error $E^R_{\mathcal O}(\phi)= \sup_{t\in [0,T]}\norm{\dot \phi(t) -F\circ \phi(t)}$. In (\ref{eq:inv_est_test}), we over-approximate $E^R_{\mathcal O}(\phi)$ using the maximum ``observed" ODE residual error, enhanced by an error bound employing the Lipschitz constant $L_R$. The Lipschitz constant $L_R$ can be further over-approximated as $L_{\dot \phi} + L L_\phi$, where $L_{\dot \phi}$ and $L_{\phi}$ are the Lipschitz constants for $\dot\phi$ and $\phi$, respectively. Generally, determining the Lipschitz constants for a neural network is challenging, particularly for multi-layered networks. Section \ref{sec:Estimation of Lipschitz Constants} provides analytical estimates for function composition useful for this task. In practice, we can also directly compute upper bounds of the true residual error $E^R_{\mathcal O}(\phi)$, e.g., using a satisfiability modulo theories (SMT) solver.
\end{remark}

\begin{remark}\label{rem:validated_sol_ivp}
As can be seen from its proof, Proposition \ref{prop:ivp} relies on Gronwall's inequality to estimate the error between the true and approximate solutions to the IVP. In practice, Gronwall's inequality, using a uniform Lipschitz constant, cannot offer useful estimates unless the time duration \([0, T]\) is very small. To illustrate, consider a moderately sized Lipschitz constant of 5 with a time duration \([0, 2]\). Then the term \(e^{2\times 5} = e^{10}\) in the error estimate would render the bound practically useless. One potential approach to overcome this is to develop validated bounds of neural network solutions that: 1) use a Lipschitz constant matrix instead of a constant; 2) estimate this Lipschitz constant matrix on a smaller domain that is computed as a guaranteed a priori enclosure of the solution on a smaller time interval; 3) continue this process until a given time interval is reached. This can be a potential area of future investigation. We refer the readers to \cite[Chapter 4.3]{li2022formal} for validated computation of solutions to IVPs. The numerical example in Section \ref{sec:numerical} was implemented using such an approach.
\end{remark}




\subsection{Elliptic equations}

In this subsection, we consider a class of elliptic equations that are slightly more general than the Laplace equation. 

\begin{proposition}\label{prop:elliptic}
    Consider a neural approximation $\phi$ to the solution of the elliptic equation
    \begin{equation}
        \label{eq:elliptic}
        \Delta \phi + c\phi = F\quad \text{in }\Omega, 
    \end{equation}
    where \( \Omega \) is a bounded domain in \( \R^n \), \( c \) and \( F \) are continuous on \( \bar\Omega \) with \( c \le 0 \). Let \( h \) be a continuous function on \( \partial \Omega \) defining the Dirichlet boundary condition
    \begin{equation}
        \label{eq:elliptic:boundary}
        \phi = h\quad \text{on }\partial\Omega. 
    \end{equation}
    Pick $O:\,C^2(\Omega)\cap C^1(\bar\Omega)\times \Omega\rightarrow\R^2$ to be
    \[
    O(\phi,x)= 
    \begin{pmatrix}
    \Delta \phi(x) + c(x)\phi(x) -F(x)\\
        {\phi(p)-h(p)}
        \end{pmatrix}
    \]
        where $x\in \Omega$ and $p=proj_{\partial \Omega}(x)$ is the projection of $x$ onto the boundary $\partial\Omega$.
        Then $O$ is continuous. Suppose that $\varphi$ is a  true solution. Assume that $\phi,\varphi \in C^2(\Omega)\cap C(\bar\Omega)$. Then 
    \begin{equation}\label{eq:poisson_est}
    E_{\varphi}(\phi)=\sup_{\bar\Omega}\abs{\phi - \varphi} \le E^B_{\mathcal{O}}(\phi) + C E^R_{\mathcal{O}}(\phi),
    \end{equation}
    where $E^B_{\mathcal{O}}(\phi)=\sup_{\partial\Omega}\abs{\phi - h}$ and $E^R_{\mathcal{O}}(\phi)=\sup_{x\in\Omega}\abs{\Delta \phi(x) + c(x)\phi(x) -F(x)}$, and $C>0$ only depends on $n$ and the diameter of $\Omega$. For example, if $\Omega\subset \mathcal{B}_R(x_0)$ for some $x_0\in\R^n$ and $R>0$, we can take $C=\frac{R^2}{2n}$. 
In particular, the estimate (\ref{eq:poisson_est}) also shows that \(O\) has a unique solution. 
\end{proposition}

\begin{proof}
    The proof directly uses an a priori estimate for this class of elliptic PDEs \cite{han2011basic}. Let $u = \phi - \varphi$. Then $u$ satisfies 
    $$
    \Delta u + c u = \Delta \phi + c\phi - F - \Delta \varphi - c\varphi + F = O(\phi,\cdot)\quad  \text{in }\Omega, 
    $$
    and
    $$
    u = \phi - h \quad  \text{on }\partial \Omega.
    $$
    Applying a standard a priori estimate for this PDE \cite[Theorem 4.3.12]{han2011basic} gives the conclusion. 
\end{proof}

\begin{remark}\label{rem:est2}
    Similar to Remark \ref{rem:estimate}, one can use the Lipschitz constants of $\phi$, $h$, and the residual $\Delta \phi(x) + c(x)\phi(x) - F(x)$ along with test errors evaluated at test points obtained from uniform gridding to over-approximate the right-hand side of (\ref{eq:poisson_est}). One may also use an SMT solver to directly obtain an upper bound of the right-hand side of (\ref{eq:poisson_est}).
\end{remark}

\subsection{Parabolic equations}

A similar result can be obtained for a class of parabolic equations. 

\begin{proposition}\label{prop:parabolic}
    Let $\Omega$ be a bounded domain in \( \R^n \) and $T$ be a positive constant. Define 
    $$
    \Omega_T = \Omega \times [0,T] = \setB{(x,t):\,x\in\Omega, \;t\in [0,T]}. 
    $$
    Let $c$ and $F$ be continuous on \( \bar\Omega_T \) with $c\ge -c_0$ for some nonnegative constant $c_0$. Let \( h \) be a continuous function on \( \partial \Omega \times [0,T] \) and $\phi_0$ be a continuous function on \( \bar\Omega \). 
    
    Consider a neural approximation $\phi$ to the solution of the following initial/boundary value problem 
    \begin{equation}
        \label{eq:parabolic}
        \begin{aligned}
            \phi_t - \Delta \phi + c\phi & = F\quad \text{in }\Omega_T, \\
            \phi(\cdot,0) & = \phi_0 \quad \text{on }\Omega,\\
            \phi & = h \quad \text{on }\partial\Omega\times [0,T]. 
        \end{aligned}
    \end{equation}
    Pick $O:\,C^{2,1}(\Omega_T)\times\bar\Omega_T\rightarrow\R^3$ to be
    \[
    O(\phi,\correction{(x,t)})= 
        \begin{pmatrix}
            \phi_t\correction{(x,t)} - \Delta \phi\correction{(x,t)} + c\correction{(x,t)}\phi\correction{(x,t)} -F\correction{(x,t)}\\
            \phi_t(0,x) - \phi_0\correction{(x,t)}\\
            \phi_t(y) - \phi_0(y)
        \end{pmatrix},
    \]
          where $y=proj_{[0,T]\times\partial\Omega}(\correction{(x,t)})$. Then $O$ is continuous. 
    Suppose that $\varphi$ is the true solution. Assume that $\phi,\varphi \in C^{2,1}(\Omega_T)\cap C(\bar\Omega_T)$. Then 
    \begin{equation}
    E_{\varphi}(\phi)=\sup_{\bar\Omega}\abs{\phi - \varphi} \le e^{c_0T}\max\left(E^B_{\mathcal{O}}(\phi),\, E^I_{\mathcal{O}}(\phi)\right)+ C E^R_{\mathcal{O}}(\phi),		\label{eq:heat_est}    
    \end{equation}
    where 
$E^B_{\mathcal{O}}(\phi)=\sup_{\partial \Omega \times [0,T]} \abs{\phi - h}$, $E^I_{\mathcal{O}}(\phi)=\sup_{\bar\Omega}\abs{\phi(\cdot,0)-\phi_0}$, 
and $E^R_{\mathcal{O}}(\phi)=\sup_{\bar \Omega_T}\abs{\phi_t - \Delta \phi + c\phi -F}$, and   
$C>0$ can be taken as $\frac{e^{c_0T}-1}{c_0}$ if $c_0>0$ or $T$ if $c_0=0$.
\end{proposition}

\begin{proof}
    The proof again relies on an a priori estimate for this class of parabolic PDEs \cite{han2011basic}. The choice of the constant $C$ is slightly optimized compared with \cite[Theorem 5.3.8]{han2011basic}; hence for completeness we provide a direct proof based on the weak maximum principle for this parabolic PDE. The argument follows closely that in \cite[Theorem 5.3.8]{han2011basic}. 
    
    Let $u = \phi - \varphi$. Then $u$ satisfies 
    $$
    Lu \coloneqq u_t - \Delta u + c u = (\phi_t - \Delta \phi + c\phi - F) - (\varphi_t - \Delta \varphi + c\varphi - F) = O(\phi,\cdot)\quad  \text{in }\Omega_T, 
    $$
    $$
    u(\cdot,0) = \phi(\cdot,\correction{0}) - \phi_0 \quad  \text{on } \Omega,
    $$
    and
    $$
    u = \phi - \varphi = \phi - h \quad  \text{on }\partial \Omega. 
    $$
    Let 
    $$
    B = \max\left(E^B_{\mathcal{O}}(\phi),\, E^I_{\mathcal{O}}(\phi)\right),\quad E = E^R_{\mathcal{O}}(\phi). 
    $$
    It follows that $L(\pm u)\le E$ in $\Omega_T$ and $\pm u\le B$ on $(\bar\Omega \times \setB{0}) \cup (\partial \Omega \times [0,T])$. 
    
    Consider the auxiliary function 
    $$
    v(x,t) = \begin{cases}
        e^{c_0t}B + \frac{e^{c_0t}-1}{c_0} F,\quad c_0>0,\\
        B +  F t,\quad c_0=0.
    \end{cases}
    $$
    We have
    \begin{align*}
        Lv &= c_0 e^{c_0t} B + e^{c_0 t} F + c \left(e^{c_0t}B + \frac{e^{c_0t}-1}{c_0} F\right)\\
        & = (c_0 + c)e^{c_0t} B  + e^{c_0 t} F  + c\frac{e^{c_0t}-1}{c_0} F \\
        & \ge e^{c_0 t} F + F - e^{c_0 t} F = F,
    \end{align*}
    where simply used the fact that $c\ge -c_0$ (and $c_0+c\ge 0$). The same holds trivially when $c_0=0$. Furthermore, $v\ge B$ on $(\bar\Omega \times \setB{0}) \cup (\partial \Omega \times [0,T])$. It follows that $L(\pm u)\le Lv$ in $\Omega_T$ and $\pm u \le v$ on $(\bar\Omega \times \setB{0}) \cup (\partial \Omega \times [0,T])$. By the maximum principle \cite[Corollary 5.3.4]{han2011basic}, we have $\abs{u}\le v$ in $\Omega_T$ and the conclusion follows. 
\end{proof}

\begin{remark}
    Similar to Remarks \ref{rem:estimate} and \ref{rem:est2}, we can estimate the right-hand side of (\ref{eq:heat_est}) after training the neural network to obtain an upper bound of the true solution error. 
\end{remark}

\subsubsection{A Generalization Bound in an $L_2$ Setting}

\begin{proposition}\label{prop:parabolic2}
    Let $\Omega$ be a bounded domain in \( \R^n \) and $T$ be a positive constant. Define 
    $$
    \Omega_T = \Omega \times [0,T] = \setB{(x,t):\,x\in\Omega, \;t\in [0,T]}. 
    $$
    Let $c,\alpha$ and $F$ be continuous on \( \bar\Omega_T \) with $c\ge -c_0$ for some nonnegative constant $c_0$ and $\alpha\ge \alpha_0$ for some constant $\alpha_0>0$. Let \( h \) be a continuous function on \( \partial \Omega \times [0,T] \) and $\phi_0$ be a continuous function on \( \bar\Omega \). 
    
    Let $\varphi$ be the true solution to the following initial/boundary value problem 
    \begin{equation}
        \begin{aligned}
            \varphi_t - \alpha\Delta \varphi + c\varphi & = F\quad \text{in }\Omega_T, \\
            \varphi(\cdot,0) & = \varphi_0 \quad \text{on }\Omega,\\
            \varphi & = h \quad \text{on }\partial\Omega\times [0,T]. 
        \end{aligned}
    \end{equation}
    Let $\phi$ be a neural approximation to the solution of the initial/boundary value problem, satisfying
    \begin{equation}
        \begin{aligned}
            \phi_{t} - \alpha\Delta \phi + c\phi = F+\varepsilon_1\quad &\text{in }\Omega_T, \\
            \phi(\cdot,0) = \correction{\varphi_0}+\varepsilon_2 \quad &\text{on }\Omega,\\
            \phi = h+\varepsilon_3 \quad &\text{on }\partial\Omega\times [0,T].
        \end{aligned}
    \end{equation}
    Assume the error functions satisfy 
    \[\varepsilon_1(x,t)\in C^1(\Omega_T)\cap L_2(\Omega_T),\quad\varepsilon_2(x)\in C^1(\Omega)\cap L_2(\Omega),\quad\varepsilon_3(x,t) \in C^2\big([0,T]; H^{1/2}(\partial\Omega)\cap L_\infty(\partial\Omega)\big)\]
    
    Let $p\in\Omega_T$ and let $O:\,C^{2,1}(\Omega_T)\times\bar\Omega_T\rightarrow\R^4$ be
    \[O(\phi,p)=\begin{pmatrix}
        \epsilon_1(p)
        \\\epsilon_2(\mathrm{proj}_\Omega(p))
        \\\epsilon_3(\mathrm{proj}_{\partial\Omega\times [0,T]}(p))
        \\(\epsilon_3)_t(\mathrm{proj}_{\partial\Omega\times [0,T]}(p))
    \end{pmatrix}\]
    Due to continuity of the error terms and the projection operator, $O$ is continuous.    
    Assume that $\phi,\varphi \in C^{2,1}(\Omega_T)\cap C(\bar\Omega_T)$. Then 
    \begin{equation}
    \begin{aligned}
        \|u(\cdot,t)\|_{L_2(\Omega)}^2
        \leq &~2e^{(2c_0+1)t}\left(\|\varepsilon_2\|_{L_2(\Omega)}^2+2|\Omega|\|\varepsilon_3(\cdot,0)\|_{L_\infty(\partial\Omega)}^2\right)\\
        &+2e^{(2c_0+1)t}\left(\int_0^te^{-(2c_0+1)s}[\|\varepsilon_1(\cdot,s)\|_{L_2(\Omega)}^2+|\Omega|\|(\varepsilon_3)_t(\cdot,s)\|_{L_\infty(\partial\Omega)}^2+|\Omega|\|\varepsilon_3(\cdot,s)\|_{L_\infty(\partial\Omega)}^2]ds\right).
    \end{aligned}
    \end{equation}
\end{proposition}

\begin{proof}
    Let $r(\cdot,t)=\xi(\varepsilon_3(\cdot,t))$, where $\xi(\cdot)$ is the harmonic extension operator. As proved in Lemma \ref{lem:harmonic_extension}, $r(x,t)$ is twice differentiable with respect to $t$.
    
    Let 
    \[v:=u-r,\quad f:=\varepsilon_1-r_{t}+\underbrace{\Delta r}_{=0}-cr,\quad\tilde\varepsilon_2:=\varepsilon_2-r.\]
    $v$ solves the following IVBVP:
    \begin{equation}
        \begin{aligned}
            v_{tt} - \alpha\Delta v + cv = f\quad &\text{in }\Omega_T, \\
            v(\cdot,0) = \tilde\varepsilon_2 \quad &\text{on }\Omega,\\
            v = 0 \quad &\text{on }\partial\Omega\times [0,T]. 
        \end{aligned}
    \end{equation}
    Multiply both sides of the PDE by $v$ and integrate over $\Omega$.
    \[
    \frac{1}{2}\frac{d}{dt}\|v\|_{L_2(\Omega)}^2+\|\alpha^{1/2}\nabla v\|_{L_2(\Omega)}^2+\int_{\Omega}cv^2=\int_{\Omega}fv.
    \]
    Using $c\geq-c_0$, and using Young's inequality to bound $\langle f,v\rangle_{L_2(\Omega)}$ gives the following upper bound
    \[
    \frac{1}{2}\frac{d}{dt}\|v\|_{L_2(\Omega)}^2+\|\alpha^{1/2}\nabla v\|_{L_2(\Omega)}^2\leq c_0\|v\|_{L_2(\Omega)}^2+\frac{1}{2}\|f\|_{L_2(\Omega)}^2+\frac{1}{2}\|v\|_{L_2(\Omega)}^2.
    \]
    Finally, we can drop the $\|\nabla v\|_{L_2(\Omega)}^2$ term to obtain a differential inequality
    \[
    \frac{d}{dt}\|v\|_{L_2(\Omega)}^2\leq (2c_0+1)\|v\|_{L_2(\Omega)}^2+\|f\|_{L_2(\Omega)}^2.
    \]
    By Gronwall,
    \begin{equation}
        \|v(\cdot,t)\|_{L_2(\Omega)}^2\leq e^{(2c_0+1)t}\left(\|v(\cdot,0)\|_{L_2(\Omega)}^2+\int_0^te^{-(2c_0+1)s}\|f(\cdot,s)\|_{L_2(\Omega)}^2ds\right).
    \end{equation}
    Expanding each term gives
    \begin{align*}
        \|v(\cdot,0)\|_{L_2(\Omega)}^2
        &=\|\tilde\varepsilon_2\|_{L_2(\Omega)}^2\\
        &=\|\varepsilon_2-r(\cdot,0)\|_{L_2(\Omega)}^2\\
        &\leq 2\|\varepsilon_2\|_{L_2(\Omega)}^2+2\|r(\cdot,0)\|_{L_2(\Omega)}^2\\
        &\leq 2\|\varepsilon_2\|_{L_2(\Omega)}^2+2|\Omega|\|\varepsilon_3(\cdot,0)\|_{L_\infty(\partial\Omega)}^2.
    \end{align*}
    \begin{align*}
        \|f(\cdot,t)\|_{L_2(\Omega)}^2
        &=\|\varepsilon_1-r_{t}-cr\|_{L_2(\Omega)}^2\\
        &\leq 2\|\varepsilon_1\|_{L_2(\Omega)}^2+2\|r_t\|_{L_2(\Omega)}^2+2\|cr\|_{L_\infty(\partial\Omega)}^2\\
        &\leq 2\|\varepsilon_1\|_{L_2(\Omega)}^2+2|\Omega|(\|(\varepsilon_3)_t(\cdot,t)\|_{L_\infty(\partial\Omega)}^2+\|\varepsilon_3(\cdot,t)\|_{L_\infty(\partial\Omega)}^2).
    \end{align*}
    This gives us the following bound
    \begin{equation}
    \begin{aligned}
        \|u(\cdot,t)\|_{L_2(\Omega)}^2
        \leq &~2e^{(2c_0+1)t}\left(\|\varepsilon_2\|_{L_2(\Omega)}^2+2|\Omega|\|\varepsilon_3(\cdot,0)\|_{L_\infty(\partial\Omega)}^2\right)\\
        &+2e^{(2c_0+1)t}\left(\int_0^te^{-(2c_0+1)s}[\|\varepsilon_1(\cdot,s)\|_{L_2(\Omega)}^2+|\Omega|\|(\varepsilon_3)_t(\cdot,s)\|_{L_\infty(\partial\Omega)}^2+|\Omega|\|\varepsilon_3(\cdot,s)\|_{L_\infty(\partial\Omega)}^2]ds\right).
    \end{aligned}
    \end{equation}
\end{proof}

Using \cref{prop:parabolic2}, we compute tighter bounds than presented in Table \ref{tab:heat1d} by computing the $L_2$ norms (or integrated) norms of some error terms instead of their uniform norms across the whole domain. We refer to \cref{sec:verification} for details on computing the error terms, and present the results below.
\begin{table*}[ht]
    \centering
    \begin{tabular}{@{}lllllll@{}} 
        \toprule
        \makecell{$\int e^{-s}\|\varepsilon_3\|^2_{L_\infty(\partial\Omega)}$} & \makecell{$\|\varepsilon_3(\cdot,0)\|_{L_\infty(\partial\Omega)}$} & \makecell{$\int e^{-s}\|(\varepsilon_3)_t\|^2_{L_\infty(\partial\Omega)}$} & \makecell{$\|\varepsilon_2\|_{L_2(\Omega)}$} & \makecell{$\int e^{-s}\|\varepsilon_1\|^2_{L_2(\Omega)}$} & \makecell{Gen. bound\\(Prop. \ref{prop:parabolic2})} & \makecell{Ref.\\error} \\ 
        \midrule
        1.81E-8 & 2.13E-6 & 3.73E-8 & 2.16E-4 & 2.26E-6 & 3.63E-3 & 2.66E-3 \\ 
        (51.03s) & (1.62s) & (58.10s) & (34.01s) & (3947s) & N/A & (951s)\\
        \bottomrule
    \end{tabular}
    \caption{Verification of generalization error for 1D heat equation: all $L_2$ errors are verified using autoLiRPA \cite{xu2020automatic}.}
    \label{tab:heat2}
\end{table*}

\begin{remark}
    Assume further that all the error terms are bounded in the $L_{\infty}(\Omega_T)$ or $L_{\infty}(\partial\Omega\times[0,T])$ norms, and let $O$ be equipped with the respective $L_{\infty}$ norms componentwise.
    \begin{equation}
    \begin{aligned}
        \|u(\cdot,t)\|_{L_2(\Omega_T)}^2
        \leq &~2e^{(2c_0+1)t}\left(\|\varepsilon_2\|_{L_2(\Omega)}^2+2|\Omega|\|\varepsilon_3(\cdot,0)\|_{L_\infty(\partial\Omega)}^2\right)\\
        &+2e^{(2c_0+1)t}\left(\|\varepsilon_1\|_{L_2(\Omega_T)}^2+|\Omega|\|(\varepsilon_3)_t\|_{L_\infty(\partial\Omega\times[0,T])}^2+|\Omega|\|\varepsilon_3\|_{L_\infty(\partial\Omega\times[0,T])}^2\right)\left(\int_0^te^{-(2c_0+1)s}ds\right)\\
        \leq &~2e^{(2c_0+1)t}|\Omega|\left(\|\varepsilon_2\|_{L_\infty(\Omega)}^2+2\|\varepsilon_3\|_{L_\infty(\partial\Omega\times[0,T])}^2\right)\\
        &+\frac{2e^{(2c_0+1)t}}{2c_0+1}|\Omega|\left(\|\varepsilon_1\|_{L_\infty(\Omega_T)}^2+|\|(\varepsilon_3)_t\|_{L_\infty(\partial\Omega\times[0,T])}^2+\|\varepsilon_3\|_{L_\infty(\partial\Omega\times[0,T])}^2\right).
    \end{aligned}
    \end{equation}
    Additionally, this shows that
    \[
    \|u(\cdot,t)\|_{L_2(\Omega_T)}^2\leq 2e^{(2c_0+1)t}\max\left(2,\frac{1}{2c_0+1}\right)|\Omega|\|O(u,\cdot)\|_O^2,
    \]
    thus verifying the assumptions of Theorem \ref{thm:generalization}.
\end{remark}

\subsection{Hyperbolic equations}

\begin{proposition}\label{prop:hyperbolic}
    Let $\Omega$ be a bounded domain in \( \R^n \) and $T$ be a positive constant. Define 
    $$
    \Omega_T = \Omega \times [0,T] = \setB{(x,t):\,x\in\Omega, \;t\in [0,T]}. 
    $$
    Let $F$ be continuous on \( \bar\Omega_T \). Let $\alpha>0$ and $c\geq 0$ be constants. Let \( h \) be a continuous function on \( \partial \Omega \times [0,T] \) and $\varphi_0,\psi_0$ be continuous functions on \( \bar\Omega \). 
    
    Let $\varphi$ be the true solution to the following initial/boundary value problem
    \begin{equation}
        \label{eq:hyperbolic}
        \begin{aligned}
            \varphi_{tt} - \alpha\Delta \varphi + c\varphi = F\quad &\text{in }\Omega_T, \\
            \varphi(\cdot,0) = \varphi_0 \quad &\text{on }\Omega,\\
            \varphi_t(\cdot,0) = \psi_0 \quad &\text{on }\Omega,\\
            \varphi = h \quad &\text{on }\partial\Omega\times [0,T]. 
        \end{aligned}
    \end{equation}
    Let $\phi$ be a neural approximation to the solution of the initial/boundary value problem, satisfying
    \begin{equation}
        \begin{aligned}
            \phi_{tt} - \alpha\Delta \phi + c\phi = F+\varepsilon_1\quad &\text{in }\Omega_T, \\
            \phi(\cdot,0) = \correction{\varphi_0}+\varepsilon_2 \quad &\text{on }\Omega,\\
            \phi_t(\cdot,0) = \psi_0+\varepsilon_3 \quad &\text{on }\Omega,\\
            \phi = h+\varepsilon_4 \quad &\text{on }\partial\Omega\times [0,T].
        \end{aligned}
    \end{equation}
    Assume error functions satisfy
    \[
    \varepsilon_1(x,t)\in C^1(\Omega_T)\cap L_2(\Omega_T)\quad
    \varepsilon_2(x),
    \varepsilon_3(x)\in C^1(\Omega)\cap L_2(\Omega)\quad
    \varepsilon_4(x)\in C^2\big([0,T]; H^{1/2}(\partial\Omega)\cap L_\infty(\partial\Omega)\big).
    \]
    Let $p\in\Omega_T$ and let $O:\,C^{2,1}(\Omega_T)\times\bar\Omega_T\rightarrow\R^7$ be
    \[O(\phi,p)=\begin{pmatrix}
        \epsilon_1(p)
        \\\epsilon_2(\mathrm{proj}_\Omega(p))
        \\\|\nabla\epsilon_2(\mathrm{proj}_\Omega(p))\|
        \\\epsilon_3(\mathrm{proj}_\Omega(p))
        \\\epsilon_4(\mathrm{proj}_{\partial\Omega\times [0,T]}(p))
        \\(\epsilon_4)_t(\mathrm{proj}_{\partial\Omega\times [0,T]}(p))
        \\(\epsilon_4)_{tt}(\mathrm{proj}_{\partial\Omega\times [0,T]}(p))
    \end{pmatrix}\]
    Due to continuity of the error terms and the projection operator, $O$ is continuous.
    Assume that $\phi,\varphi \in C^{2,1}(\Omega_T)\cap C(\bar\Omega_T)$. Then 
    \begin{equation}
    \begin{aligned}
        \|\phi(\cdot,t)-\varphi(\cdot,t)\|_{L_2(\Omega)}
        \leq&~
        \frac{1}{\sqrt{2c}}(\|\varepsilon_3\|_{L_2(\Omega)}+|\Omega|^{1/2}\|(\varepsilon_4)_{t}(\cdot,0)\|_{L_\infty(\partial\Omega)})\\
        &+\frac{1}{\sqrt{2}}(\|\varepsilon_2\|_{L_2(\Omega)}+|\Omega|^{1/2}\|\varepsilon_4(\cdot,0)\|_{L_\infty(\partial\Omega)})
        +\sqrt\frac{2\alpha}{c}\|\nabla\varepsilon_2\|_{L_2(\Omega)}\\
        &+\frac{1}{\sqrt{2c}}\int_0^t\|\varepsilon_1(\cdot,s)\|_{L_2(\Omega)}+|\Omega|^{1/2}\|(\varepsilon_4)_{tt}(\cdot,s)\|_{L_\infty(\Omega)}+c|\Omega|^{1/2}\|\varepsilon_4(\cdot,s)\|_{L_\infty(\Omega)}ds\\
        &+|\Omega|^{1/2}\|\varepsilon_4(\cdot,t)\|_{L_\infty(\partial\Omega)},
    \end{aligned}
    \end{equation}
\end{proposition}

\begin{proof}
    Assume the neural approximation error of $\phi$ is governed by functions $\varepsilon_1(x,t),\varepsilon_2(x),\varepsilon_3(x),\varepsilon_4(x,t)$, which are all bounded in the $L_2(\Omega)$ norms. Moreover, assume $\varepsilon_4 \in C^2\big([0,T]; H^{1/2}(\partial\Omega)\cap L_\infty(\partial\Omega)\big)$. 
    Let $u = \phi - \varphi$. Then $u$ satisfies 
    \begin{equation}
        \begin{aligned}
            u_{tt} - \alpha\Delta u + cu = \varepsilon_1\quad &\text{in }\Omega_T, \\
            u(\cdot,0) = \varepsilon_2 \quad &\text{on }\Omega,\\
            u_t(\cdot,0) = \varepsilon_3 \quad &\text{on }\Omega,\\
            u = \varepsilon_4 \quad &\text{on }\partial\Omega\times [0,T]. 
        \end{aligned}
    \end{equation}
    Let $r(\cdot,t)=\xi(\varepsilon_4(\cdot,t))$, where $\xi(\cdot)$ is the harmonic extension operator. As proved in Lemma \ref{lem:harmonic_extension}, $r(x,t)$ is twice differentiable with respect to $t$.
    
    Let 
    \[v:=u-r,\quad f:=\varepsilon_1-r_{tt}+\alpha\underbrace{\Delta r}_{=0}-cr,\quad\tilde\varepsilon_2:=\varepsilon_2-r,\quad\tilde\varepsilon_3:=\varepsilon_3-r_t.\]
    $v$ solves the following IVBVP:
    \begin{equation}
        \begin{aligned}
            v_{tt} - \alpha\Delta v + cv = f\quad &\text{in }\Omega_T, \\
            v(\cdot,0) = \tilde\varepsilon_2 \quad &\text{on }\Omega,\\
            v_t(\cdot,0) = \tilde\varepsilon_3 \quad &\text{on }\Omega,\\
            v = 0 \quad &\text{on }\partial\Omega\times [0,T]. 
        \end{aligned}
    \end{equation}
    Multiplying the PDE term with $v_t$ gives
    \begin{equation}
        v_tv_{tt} - v_t\alpha\Delta v + cvv_t = fv_t.
    \end{equation}
    Define the energy function
    \begin{equation}
        E_v(t)=\frac{1}{2}\int_\Omega v_t^2+\alpha\|\nabla v\|^2+cv^2dx.
    \end{equation}
    The time derivative of $E(t)$ is
    \begin{align*}
        E_v'(t)
        &=\int_\Omega v_tv_{tt}+\alpha\nabla v\cdot\nabla v_t+cvv_tdx\\
        &=\int_\Omega v_tv_{tt}-v_t\alpha\Delta v+cvv_tdx+\oint_{\partial\Omega}\alpha\underbrace{v_t}_{=0}\nabla v\cdot\mathbf{n}\\
        &=\int_\Omega fv_tdx.
    \end{align*}
    Thus,
    \begin{equation}
        E_v'(t)\leq\|f\|_{L_2(\Omega)}\|v_t\|_{L_2(\Omega)}.
    \end{equation}
    Using $\|v_t\|_{L_2(\Omega)}\leq\sqrt{2E_v(t)}$, we have
    \begin{equation*}
        E_v'(t)\leq\sqrt{2}\|f\|_{L_2(\Omega)}\sqrt{E_v(t)},
    \end{equation*}
    with initial condition
    \[E_v(0)=\frac{1}{2}\|\tilde\varepsilon_3\|_{L_2(\Omega)}^2+\frac{\alpha}{2}\|\nabla\tilde\varepsilon_2\|_{L_2(\Omega)}^2+\frac{c}{2}\|\tilde\varepsilon_2\|_{L_2(\Omega)}^2.\]
    Let $G_v(t):=\sqrt{E_v(t)}$.
    \[
    G_v'(t)=\frac{E_v'(t)}{2\sqrt{E_v(t)}}\leq\frac{1}{\sqrt{2}}\|f(\cdot,t)\|_{L_2(\Omega)}.
    \]
    Integrate from $0$ to $t$
    \[
    \sqrt{E_v(t)}=G_v(t)\leq\sqrt{E_v(0)}+\frac{1}{\sqrt{2}}\int_0^t\|f(\cdot,s)\|_{L_2(\Omega)}ds.
    \]
    Using $(a+b)^2\leq 2a^2+2b^2$,
    \begin{align*}
    E_v(t)
    &\leq 2E_v(0)+\int_0^t\|f(\cdot,s)\|_{L_2(\Omega)}^2ds\\
    &=\|\tilde\varepsilon_3\|_{L_2(\Omega)}^2+\alpha\|\nabla\tilde\varepsilon_2\|_{L_2(\Omega)}^2+c\|\tilde\varepsilon_2\|_{L_2(\Omega)}^2+\int_0^t\|f(\cdot,s)\|_{L_2(\Omega)}^2ds.
    \end{align*}
    We finally revert back to bounding $\|u(\cdot,t)\|_{L_2(\Omega)}$, using
    \[
    \|u(\cdot,t)\|_{L_2(\Omega)}\leq\|v(\cdot,t)\|_{L_2(\Omega)}+\|r(\cdot,t)\|_{L_2(\Omega)}.
    \]
    We can bound $\|r(\cdot,t)\|_{L_2(\Omega)}$ using the maximum principle
    \[
    \|r(\cdot,t)\|_{L_2(\Omega)}\leq|\Omega|^{1/2}\|\varepsilon_4(\cdot,t)\|_{L_\infty(\partial\Omega)},
    \]
    and we can bound $\|v(\cdot,t)\|_{L_2(\Omega)}$ using $E_v(t)$,
    \begin{align*}
    \|v(\cdot,t)\|_{L_2(\Omega)}
    &\leq\frac{1}{\sqrt{c}}\sqrt{E_v(t)}\\
    &\leq\sqrt{\frac{E_v(0)}{c}}+\frac{1}{\sqrt{2c}}\int_0^t\|f(\cdot,s)\|_{L_2(\Omega)}ds\\
    &\leq\sqrt{\frac{E_v(0)}{c}}+\frac{1}{\sqrt{2c}}\int_0^t\|f(\cdot,s)\|_{L_2(\Omega)}ds\\
    &=\sqrt{\frac{1}{2c}\|\tilde\varepsilon_3\|_{L_2(\Omega)}^2+\frac{\alpha}{2c}\|\nabla\tilde\varepsilon_2\|_{L_2(\Omega)}^2+\frac{1}{2}\|\tilde\varepsilon_2\|_{L_2(\Omega)}^2}+\frac{1}{\sqrt{2c}}\int_0^t\|f(\cdot,s)\|_{L_2(\Omega)}ds\\
    &\leq\frac{1}{\sqrt{2c}}\|\tilde\varepsilon_3\|_{L_2(\Omega)}+\sqrt\frac{\alpha}{2c}\|\nabla\tilde\varepsilon_2\|_{L_2(\Omega)}+\frac{1}{\sqrt{2}}\|\tilde\varepsilon_2\|_{L_2(\Omega)}+\frac{1}{\sqrt{2c}}\int_0^t\|f(\cdot,s)\|_{L_2(\Omega)}ds.
    \end{align*}
    Finally, each of the terms can be bounded in terms of $\varepsilon_i$ for $i=1,2,3,4$, through the maximum principle, since $r$ and $r_{tt}$ are harmonic extensions of $\varepsilon_4$ and $(\varepsilon_4)_{tt}$ respectively.
    \begin{align*}
    \|f(\cdot,s)\|_{L_2(\Omega)}
    &\leq\|\varepsilon_1(\cdot,s)\|_{L_2(\Omega)}+\|r_{tt}(\cdot,s)\|_{L_2(\Omega)}+\|cr(\cdot,s)\|_{L_2(\Omega)}\\
    &\leq\|\varepsilon_1(\cdot,s)\|_{L_2(\Omega)}+|\Omega|^{1/2}\|(\varepsilon_4)_{tt}(\cdot,s)\|_{L_\infty(\partial\Omega)}+c|\Omega|^{1/2}\|\varepsilon_4(\cdot,s)\|_{L_\infty(\partial\Omega)},
    \end{align*}
    \begin{align*}
        \|\tilde\varepsilon_3\|_{L_2(\Omega)}
        &\leq\|\varepsilon_3\|_{L_2(\Omega)}+\|r_t\|_{L_2(\Omega)}\\
        &\leq\|\varepsilon_3\|_{L_2(\Omega)}+|\Omega|^{1/2}\|(\varepsilon_4)_{t}(\cdot,0)\|_{L_\infty(\partial\Omega)},
    \end{align*}
    \begin{align*}
        \|\tilde\varepsilon_2\|_{L_2(\Omega)}
        &\leq\|\varepsilon_2\|_{L_2(\Omega)}+\|r(\cdot,0)\|_{L_2(\Omega)}\\
        &\leq\|\varepsilon_2\|_{L_2(\Omega)}+|\Omega|^{1/2}\|\varepsilon_4(\cdot,0)\|_{L_\infty(\partial\Omega)},
    \end{align*}
    \begin{align*}
        \|\nabla\tilde\varepsilon_2\|_{L_2(\Omega)}
        &\leq\|\nabla\varepsilon_2\|_{L_2(\Omega)}+\|\nabla r(\cdot,0)\|_{L_2(\Omega)}\\
        &\leq 2\|\nabla\varepsilon_2\|_{L_2(\Omega)},
    \end{align*}
    where the last line is obtained from two properties: (a) $r$ is the minimizer of the Dirichlet energy \citep{evans1998partial}:
    $\min\{\|\nabla v\|_{L_2(\Omega)}:v|_{\partial\Omega}=\varepsilon_4(\cdot,0)\}$, and (b) $\varepsilon_2$ is equal to $\varepsilon_4(\cdot,0)$ along the boundary due to continuity of $u$: $\varepsilon_2|_{\partial\Omega}=\varepsilon_4(\cdot,0)$.
    
    Putting it all together,
    \begin{equation}
    \begin{aligned}
        \|u(\cdot,t)\|_{L_2(\Omega)}
        \leq&~
        \frac{1}{\sqrt{2c}}(\|\varepsilon_3\|_{L_2(\Omega)}+|\Omega|^{1/2}\|(\varepsilon_4)_{t}(\cdot,0)\|_{L_\infty(\partial\Omega)})\\
        &+\frac{1}{\sqrt{2}}(\|\varepsilon_2\|_{L_2(\Omega)}+|\Omega|^{1/2}\|\varepsilon_4(\cdot,0)\|_{L_\infty(\partial\Omega)})
        +\sqrt\frac{2\alpha}{c}\|\nabla\varepsilon_2\|_{L_2(\Omega)}\\
        &+\frac{1}{\sqrt{2c}}\int_0^t\|\varepsilon_1(\cdot,s)\|_{L_2(\Omega)}+|\Omega|^{1/2}\|(\varepsilon_4)_{tt}(\cdot,s)\|_{L_\infty(\partial\Omega)}+c|\Omega|^{1/2}\|\varepsilon_4(\cdot,s)\|_{L_\infty(\partial\Omega)}ds\\
        &+|\Omega|^{1/2}\|\varepsilon_4(\cdot,t)\|_{L_\infty(\partial\Omega)},
    \end{aligned}
    \end{equation}
    as desired.
\end{proof}

Using \cref{prop:hyperbolic}, we compute tighter bounds than presented in Table \ref{tab:wave1d} by computing the $L_2$ norms (or integrated) norms of some error terms instead of their uniform norms across the whole domain. We refer to \cref{sec:verification} for details on computing the error terms, and present the results below.
\begin{table*}[ht]
    \centering
    \begin{tabular}{@{}lllllllll@{}} 
        \toprule
        \makecell{$\int_0^1\|\varepsilon_4\|_{L_\infty(\partial\Omega)}$} & 
        \makecell{$\|(\varepsilon_4)_t\|_{L_\infty(\partial\Omega\times[0,T])}$} & 
        \makecell{$\int_0^1\|(\varepsilon_4)_{tt}\|_{L_\infty(\partial\Omega)}$} & 
        \makecell{$\|\varepsilon_2\|_{L_\infty(\Omega)}$} & 
        \makecell{$\|\nabla\varepsilon_2\|_{L_2(\Omega)}$} & 
        \makecell{$\|\varepsilon_3\|_{L_2(\Omega)}$} & 
        \makecell{$\int_0^1\|\varepsilon_1\|_{L_2(\Omega)}$}\\ 
        \midrule
        1.01E-5 & 2.89E-5 & 2.29E-5 & 2.28E-5 & 7.42E-5 & 2.83E-5 & 7.19E-5\\ 
        (56.91s) & (52.88s) & (105.64s) & (21.62s) & (37.59s) & (28.91s) & (4764s)\\
        \bottomrule
    \end{tabular}
\end{table*}

\begin{table*}[ht]
    \centering
    \begin{tabular}{@{}llllllll@{}} 
        \toprule
        \makecell{$\|\varepsilon_4\|_{L_\infty(\partial\Omega\times[0,T])}$} & $\|\varepsilon_4(\cdot,0)\|_{L_\infty(\partial\Omega)}$ & $\|(\varepsilon_4)_t(\cdot,0)\|_{L_\infty(\partial\Omega)}$ & \makecell{Generalization error\\ (Prop. \ref{prop:hyperbolic})} & \makecell{True\\error} \\ 
        \midrule
        1.44E-5 & 7.87E-6 & 9.31E-7 & 3.67E-4 & 2.51E-4\\ 
        (29.40s) & (0.31s) & (0.51s) & N/A & (1133s)\\
        \bottomrule
    \end{tabular}
    \caption{Verification of generalization error for 1D wave equation: all errors are verified using autoLiRPA \citep{xu2020automatic} with the generalization error calculated using Proposition \ref{prop:hyperbolic}. }
    \label{tab:wave1d_2}
\end{table*}

\begin{remark}\label{rem:hyperbolic}
    Assume further that all the error terms are bounded in the $L_{\infty}(\Omega_T)$ or $L_{\infty}(\partial\Omega\times[0,T])$ norms, and let $O$ be equipped with the respective $L_{\infty}$ norms componentwise.
    \begin{equation}
    \begin{aligned}
        \frac{1}{|\Omega|^{1/2}}\|u(\cdot,t)\|_{L_2(\Omega)}
        \leq&~
        \frac{1}{\sqrt{2c}}(\|\varepsilon_3\|_{L_\infty(\Omega)}+\|(\varepsilon_4)_{t}(\cdot,0)\|_{L_\infty(\partial\Omega)})\\
        &+\frac{1}{\sqrt{2}}(\|\varepsilon_2\|_{L_\infty(\Omega)}+\|\varepsilon_4(\cdot,0)\|_{L_\infty(\partial\Omega)})
        +\sqrt\frac{2\alpha}{c}\|\nabla\varepsilon_2\|_{L_\infty(\Omega)}\\
        &+\frac{t}{\sqrt{2c}}\left(\|\varepsilon_1\|_{L_\infty(\Omega_T)}+\|(\varepsilon_4)_{tt}\|_{L_\infty(\Omega_T)}+c\|\varepsilon_4\|_{L_\infty(\Omega_T)}\right)\\
        &+\|\varepsilon_4\|_{L_\infty(\partial\Omega\times[0,T])}.
    \end{aligned}
    \end{equation}
    Additionally, this shows that
    \[
    \|u(\cdot,t)\|_{L_2(\Omega)}\leq|\Omega|^{1/2}\max\left(\frac{1}{\sqrt{2c}},\sqrt\frac{2\alpha}{c},\frac{T}{\sqrt{2c}},\frac{cT}{\sqrt{2c}},1\right)\|O(u,\cdot)\|_O,
    \]
    thus verifying the assumptions of Theorem \ref{thm:generalization}.
\end{remark}

\subsection{Allen-Cahn Equation}

\begin{proposition}\label{prop:allen-cahn}
    \correction{
    Let $\Omega=[0,L]$ and $T$ be a positive constant.
    Let $\varphi_0$ be a continuous function on \( \bar\Omega \), and let $\rho,\nu>0$ be constants. 
    Let $\varphi$ be the true solution to the following initial value problem with periodic boundary conditions
    \begin{equation}
        \label{eq:allen-cahn}
        \begin{aligned}
            \varphi_{t} + \rho\varphi(\varphi^2-1) = \nu\varphi_{xx}\quad &\text{in }\Omega_T, \\
            \varphi(\cdot,0) = \varphi_0 \quad &\text{on }\Omega.
        \end{aligned}
    \end{equation}
    Let $\phi$ be a neural approximation to the solution of the initial/boundary value problem. To enforce periodicity, consider the modification of the input space of the neural network using $x\to(\sin(2\pi x/L),\cos(2\pi x/L))$. $\phi$ satisfies
    \begin{equation}
        \begin{aligned}
            \phi_{t} + \rho\phi(\phi^2-1) = \nu\phi_{xx}+\varepsilon_1\quad &\text{in }\Omega_T, \\
            \phi(\cdot,0) = \varphi_0+\varepsilon_2 \quad &\text{on }\Omega.
        \end{aligned}
    \end{equation}
    Assume error functions satisfy
    \[
    \varepsilon_1(x,t)\in C(\Omega_T)\cap L_2(\Omega_T),\quad
    \varepsilon_2(x)\in C(\Omega)\cap L_2(\Omega)\quad
    \]
    Let $p\in\Omega_T$ and let $O:\,C^{2,1}(\Omega_T)\times\bar\Omega_T\rightarrow\R^2$ be
    \[O(\phi,p)=\begin{pmatrix}
        \epsilon_1(p)
        \\\epsilon_2(\mathrm{proj}_\Omega(p))
    \end{pmatrix}\]
    Due to continuity of the error terms and the projection operator, $O$ is continuous.
    Let $\phi$ be a neural approximation to the solution of the initial/boundary value problem.
    Assume that $\phi,\varphi \in C^{2,1}(\Omega_T)\cap C(\bar\Omega_T)$. Then 
    \begin{equation}
        \|\phi(\cdot,t)-\varphi(\cdot,t)\|_{L_2[0,L]}\leq e^{\rho t}\|\varepsilon_2\|_{L_2[0,L]}+\int_0^te^{\rho(t-s)}\|\varepsilon_1(\cdot,s)\|_{L_2[0,L]}ds.
    \end{equation}}
\end{proposition}

\begin{proof}
    Assume the neural approximation error of $\phi$ is governed by functions $\varepsilon_1(x,t),\varepsilon_2(x)$, which are all bounded in the $L_2$ norm.
    Let $u = \phi - \varphi$. Subtracting the Allen-Cahn equations gives
    \[
    u_t+\rho(\phi(\phi^2-1)-\varphi(\varphi^2-1))=\nu u_{xx}+\varepsilon_1.
    \]
    Since
    \[
    \phi(\phi^2-1)-\varphi(\varphi^2-1)=\phi^3-\varphi^3-u,
    \]
    then $u$ satisfies 
    \begin{equation}
        \begin{aligned}
            u_t-\nu u_{xx}=\varepsilon_1+\rho u-\rho(\phi^3-\varphi^3)\quad &\text{in }\Omega_T, \\
            u(\cdot,0) = \varepsilon_2 \quad &\text{on }\Omega.
        \end{aligned}
    \end{equation}
    Define the energy function
    \begin{equation}
        E(t)=\frac{1}{2}\|u(\cdot,t)\|_{L_2[0,L]}^2=\frac{1}{2}\int_0^Lu(x,t)^2dx.
    \end{equation}
    Multiply both sides of the PDE with $u$ and integrate along $[0,L]$. Periodicity removes the boundary terms.
    \begin{align*}
        \int_0^Luu_t&=\frac{1}{2}\frac{d}{dt}\|u\|^2_{L_2[0,L]}\\
        \nu\int_0^Luu_{xx}dx&=-\nu\|u_x\|^2_{L_2[0,L]}\leq 0\\
        \int_0^Lu(\phi^3-\varphi^3)dx&=\int_0^L(\phi-\varphi)(\phi^3-\varphi^3)dx=\int_0^L(\phi-\varphi)^2(\phi^2+\phi\varphi+\varphi^2)dx\geq 0
    \end{align*}
    Putting it all together, we have
    \begin{align*}
    \frac{1}{2}\frac{d}{dt}\|u\|^2_{L_2[0,L]}&+\nu\|u_x\|^2_{L_2[0,L]}+\rho\int_0^L(\phi-\varphi)^2(\phi^2+\phi\varphi+\varphi^2)dx=\rho\|u\|^2_{L_2[0,L]}+\langle u,\varepsilon_1\rangle_{L_2[0,L]},\\
    \frac{1}{2}\frac{d}{dt}\|u\|^2_{L_2[0,L]}&\leq\rho\|u\|^2_{L_2[0,L]}+\|u\|_{L_2[0,L]}\,\|\varepsilon_1\|_{L_2[0,L]},\\
    \frac{d}{dt}\|u\|_{L_2[0,L]}&\leq\rho\|u\|_{L_2[0,L]}+\|\varepsilon_1\|_{L_2[0,L]},\\
    \end{align*}
    By Gronwall,
    \begin{align*}
        \|u(\cdot,t)\|_{L_2[0,L]}
        &\leq e^{\rho t}\|u(\cdot,0)\|_{L_2[0,L]}+\int_0^t e^{\rho(t-s)}\|\varepsilon_1(\cdot,s)\|_{L_2[0,L]}ds
    \end{align*}
    as desired.
    Equivalently, if we equip $O$ with the $L_\infty$ norm, then
    \[\|u(\cdot,t)\|_{L_2[0,L]}\leq e^{\rho T}\max\left(1,\frac{1}{\rho}\right)\|O(u(\cdot),\cdot)\|_O,\]
    thus verifying the assumptions of Theorem \ref{thm:generalization}.
\end{proof}

\subsection{Linear Elasticity Equation}

While the examples considered thus far involve scalar-valued PDEs, many scientific and engineering applications are governed by systems of PDEs with vector-valued solutions. To demonstrate that our verification framework naturally extends to this setting, we consider the linear elasticity equation \citep{slaughter2002linearized}. Let $\Omega\subset\R^n$ be a convex bounded domain and let $u:\Omega\rightarrow\R^n$ denote the displacement field. The static linear elasticity equation is given by
\begin{equation}\label{eq:elasticity}
-\nabla\cdot \sigma(u)=f,
\end{equation}
where $f:\Omega\rightarrow\R^n$ is a body force and
\begin{equation}
\sigma(u)
=
2\mu \varepsilon(u)
+
\lambda \operatorname{tr}(\varepsilon(u))I
\end{equation}
is the Cauchy stress tensor. Here
\begin{equation}
\varepsilon(u)
=
\frac12\left(\nabla u+\nabla u^\top\right)
\end{equation}
denotes the strain tensor, while $\mu>0$ and $\lambda\ge 0$ are the Lamé parameters. We assume homogeneous Dirichlet boundary conditions on $\partial\Omega$.

\begin{proposition}\label{prop:elasticity}
Let $\Omega$ be a compact and convex domain in $\R^n$.
Let $\varphi\in H^1(\Omega;\R^n)$ be the weak solution of the linear elasticity equation \eqref{eq:elasticity} with Dirichlet boundary condition
\begin{equation}
    \varphi=h\quad\text{on }\partial\Omega.
\end{equation}
Let $\phi\in H^1(\Omega;\R^n)$ denote a neural approximation of the solution.
Pick $O:\,C^2(\Omega)\cap C^1(\bar\Omega)\times \Omega\rightarrow\R^2$ to be
\[
O(\phi,x)= 
\begin{pmatrix}
-\nabla\cdot\sigma(\phi(x))-f(x)\\
{\phi(p)-h(p)}
\end{pmatrix}
=:
\begin{pmatrix}
\varepsilon_1(x)\\
\varepsilon_2(p)
\end{pmatrix}
\]
where $x\in \Omega$ and $p=proj_{\partial \Omega}(x)$ is the projection of $x$ onto the boundary $\partial\Omega$.
Then $O$ is continuous.
Then
\[
|\phi-\varphi|_{H^1(\Omega)}
\le
\frac{C_K^2}{2\mu}
(C_P\|\varepsilon_1\|_{L^2(\Omega)}+(1+C_P)\|\varepsilon_2\|_{H^{1/2}(\partial\Omega)}),
\]
where $C_K$ is a Korn constant satisfying
\[
|\nabla v|_{L^2}
\le
C_K |\varepsilon(v)|_{L^2},
\qquad
v\in H_0^1(\Omega;\R^n),
\]
and $C_P$ is a Poincaré constant satisfying
\[
|v|_{L^2}
\le
C_P |\nabla v|_{L^2}.
\]
\end{proposition}

\begin{proof}
Define the error $u:=\phi-\varphi$. Then $-\nabla\cdot\sigma(u)=\varepsilon_1$ in $\Omega$ and $u|_{\partial\Omega}=\varepsilon_2$.
Let
\[
g:=u|_{\partial\Omega}=\phi-h.
\]
By the trace extension theorem \citep{evans2010partial}, there exists a lifting $\eta\in H^1(\Omega;\R^n)$ such that $\eta|_{\partial\Omega}=g$ and
\[
|\eta|_{H^1(\Omega)}
\le
C_{\rm tr}
|g|_{H^{1/2}(\partial\Omega)}
=
C_{\rm tr}
|\phi-h|_{H^{1/2}(\partial\Omega)}.
\]
We use the quotient norm on $H^{1/2}(\partial\Omega)$, defined by
\[
|g|_{H^{1/2}(\partial\Omega)}
:=
\inf_{\substack{v\in H^1(\Omega;\mathbb R^n)\\ v|_{\partial\Omega}=g}}
|v|_{H^1(\Omega)}.
\]
With this convention, there exists a lifting $\eta\in H^1(\Omega;\mathbb R^n)$ satisfying
\[
\eta|_{\partial\Omega}=g,
\qquad
|\eta|_{H^1(\Omega)}
=
|g|_{H^{1/2}(\partial\Omega)}.
\]
Thus one may take $C_{\rm tr}=1$.

Define $w:=u-\eta$. Then $w\in H_0^1(\Omega;\R^n)$ and
\[
-\nabla\cdot\sigma(w)
=
\varepsilon_1
+
\nabla\cdot\sigma(\eta).
\]
Multiplying by $w$ and integrating by parts yields
\[
\int_\Omega
\sigma(w):\varepsilon(w)\,dx
=
\int_\Omega
\varepsilon_1\cdot w\,dx
-
\int_\Omega
\sigma(\eta):\varepsilon(w)\,dx.
\]
Using the constitutive law \citep{truesdell2004non},
\[
\int_\Omega
\sigma(w):\varepsilon(w)\,dx
=
\int_\Omega
\left(
2\mu |\varepsilon(w)|^2
+
\lambda |\nabla\cdot w|^2
\right)dx
\ge
2\mu |\varepsilon(w)|_{L^2}^2.
\]
By Korn's inequality \citep{horgan1995korn},
\[
|\nabla w|_{L^2}^2
\le
C_K^2
|\varepsilon(w)|_{L^2}^2,
\]
and therefore
\[
\int_\Omega
\sigma(w):\varepsilon(w)\,dx
\ge
\frac{2\mu}{C_K^2}
|\nabla w|_{L^2}^2.
\]

Combining the above estimates and using Cauchy--Schwarz gives
\[
\frac{2\mu}{C_K^2}
|\nabla w|_{L^2}^2
\le
|\varepsilon_1|_{L^2}
|w|_{L^2}
+
|\sigma(\eta)|_{L^2}
|\varepsilon(w)|_{L^2}.
\]

Since $w\in H_0^1(\Omega;\R^n)$, Poincaré's inequality yields
\[
|w|_{L^2}
\le
C_P |\nabla w|_{L^2},
\]
while
\[
|\varepsilon(w)|_{L^2}
\le
|\nabla w|_{L^2}.
\]
Hence
\[
\frac{2\mu}{C_K^2}
|\nabla w|_{L^2}
\le
C_P |\varepsilon_1|_{L^2}
+
|\sigma(\eta)|_{L^2}.
\]

Since
\[
\sigma(\eta)
=
2\mu \varepsilon(\eta)
+
\lambda
\operatorname{tr}(\varepsilon(\eta))I,
\]
there exists a constant $C_{\lambda,\mu,n}:=2\mu+n\lambda$ such that
\[
|\sigma(\eta)|_{L^2}
\le
C_{\lambda,\mu,n}
|\eta|_{H^1(\Omega)}.
\]
Therefore,
\[
|\nabla w|_{L^2}
\le
\frac{C_K^2}{2\mu}
\left(
C_P |\varepsilon_1|_{L^2}
+
C_{\lambda,\mu,n}
|\eta|_{H^1(\Omega)}
\right).
\]

Using the lifting estimate,
\[
|\nabla w|_{L^2}
\le
\frac{C_K^2}{2\mu}
\left(
C_P |\varepsilon_1|_{L^2}
+
C_{\lambda,\mu,n}
C_{\rm tr}
|\phi-h|_{H^{1/2}(\partial\Omega)}
\right).
\]

Finally, since
\[
u=w+\eta,
\]
we obtain
\[
|u|_{H^1(\Omega)}
\le
|w|_{H^1(\Omega)}
+
|\eta|_{H^1(\Omega)}.
\]
Applying Poincaré's inequality to $w$ once more and using the lifting estimate gives
\[
|\phi-\varphi|_{H^1(\Omega)}
\le
C_1
|\varepsilon_1|_{L^2(\Omega)}
+
C_2
|\phi-h|_{H^{1/2}(\partial\Omega)},
\]
where
\[
C_1
=
\frac{C_K^2 C_P}{2\mu},
\qquad
C_2
=
C_{\rm tr}
\left(
1+
\frac{C_K^2 C_{\lambda,\mu,n}}{2\mu}
\right).
\]
This completes the proof.
\end{proof}

For convex domains contained in a ball of radius $R$, Korn and Poincaré inequalities imply
\[
C_K^2\le 2,
\qquad
C_P^2\le \frac{R^2}{2n}.
\]
Therefore,
\[
C_1
=
\frac{C_K^2C_P}{2\mu}
\le
\frac{R}{\sqrt n\,\mu},
\]
and Proposition~\ref{prop:elasticity} yields
\[
|\phi-\varphi|_{H^1(\Omega)}
\le
\frac{R}{\sqrt n\,\mu}
|\varepsilon_1|_{L^2(\Omega)}
+
C_2
|\phi-h|_{H^{1/2}(\partial\Omega)},
\]
where
\[
C_2
=
C_{\rm tr}
\left(
1+
\frac{C_K^2C_{\lambda,\mu,n}}{2\mu}
\right)
=5+
\frac{4n\lambda}{2\mu}
\]
depends only on the domain geometry and material parameters.
In particular, certified interior and boundary residual bounds immediately imply certified solution error bounds.
By the continuous embedding
\[
|\phi-h|_{H^{1/2}(\partial\Omega)}
\le
|\partial\Omega|^{1/2}
|\phi-h|_{L^\infty(\partial\Omega)}
=
|\partial\Omega|^{1/2}\|\varepsilon_2\|_{L^\infty(\Omega)},
\]
we further obtain
\[
|\phi-\varphi|_{H^1(\Omega)}
\le
\frac{R}{\sqrt n\,\mu}
\|\varepsilon_1\|_{L^2(\Omega)}
+
\left(5+
\frac{4n\lambda}{2\mu}\right)|\partial\Omega|^{1/2}\|\varepsilon_2\|_{L^\infty(\Omega)}
\]
Equivalently, if we equip $O$ with the $L_\infty$ norm, then
\[|\phi-\varphi|_{L^2(\Omega)}\leq\max\left(\frac{R}{\sqrt n\,\mu}|\Omega|,\left(5+
\frac{4n\lambda}{2\mu}\right)|\partial\Omega|^{1/2}\right)\|O(u(\cdot),\cdot)\|_O,\]
thus verifying the assumptions of Theorem \ref{thm:generalization}.

\section{Reproducibility Details}\label{sec:app:implementation}

\subsection{Van der Pol equation}

We solve the ODE using neural networks of different sizes, on different intervals of time duration, and for different choices of the parameter $\mu$. On each unit interval, we choose $1,000$ collocation points for optimizing the neural network. We tried both LBFGS and Adam as optimizers; they seem to provide comparable training results, but LBFGS takes less time. For LBFGS, we choose a learning rate of 1e-1 and run for 1,000 steps, whereas for Adam, we choose a learning rate of 1e-3 and run for 40,000 steps. On a 2 GHz Quad-Core Intel Core i5 MacBook Pro, the computation time is approximately 200 seconds for Adam and 60 seconds for LBFGS with the aforementioned settings.  

Table \ref{table:training_van_der_pol} shows the maximum residual errors (for a fixed random seed) achieved using different network sizes and optimizers. While the mean squared residual error is much lower (of order 1E-6), we choose to show maximum residual errors only, because our a posterior verification is based on maximum residual errors. The activation function is chosen to be the hyperbolic tangent function for all networks. The initial condition of solutions is set to be $(x_{10},x_{20})=(1.0,0.0)$. 

\begin{table}[ht!]
    \centering
    \begin{tabular}{ccccc}
        \toprule
        & \multicolumn{4}{c}{Network Size} \\
        \cmidrule(lr){2-5}
        & $2 \times 10$ & $2 \times 20$ & $2 \times 30$ & $3 \times 10$ \\
        \midrule
        Adam & 1.21E-1 & 9.72E-2 & 2.32E-2 & 2.28E-2  \\
        LBFGS & 9.46E-2 & 1.76E-2 & 1.56E-2 & 2.21E-2 \\
        \bottomrule
    \end{tabular}
    \caption{Solving the Van der Pol equation with $\mu=1$ over the interval $[0,7]$ using different neural networks. The displayed maximum residual is the highest tested among 70,000 points over the interval $[0,7]$.}\label{table:training_van_der_pol}
\end{table}

Figure \ref{fig:van_der_pol} shows the neural solutions compared with the fourth-order Runge-Kutta (RK4) with step size 1e-5 solved on the interval $[0,7]$. It can be seen that the neural solutions are indistinguishable from the RK4 solution.

\begin{figure}[ht]
    \centering
    \includegraphics[width=0.47\textwidth]{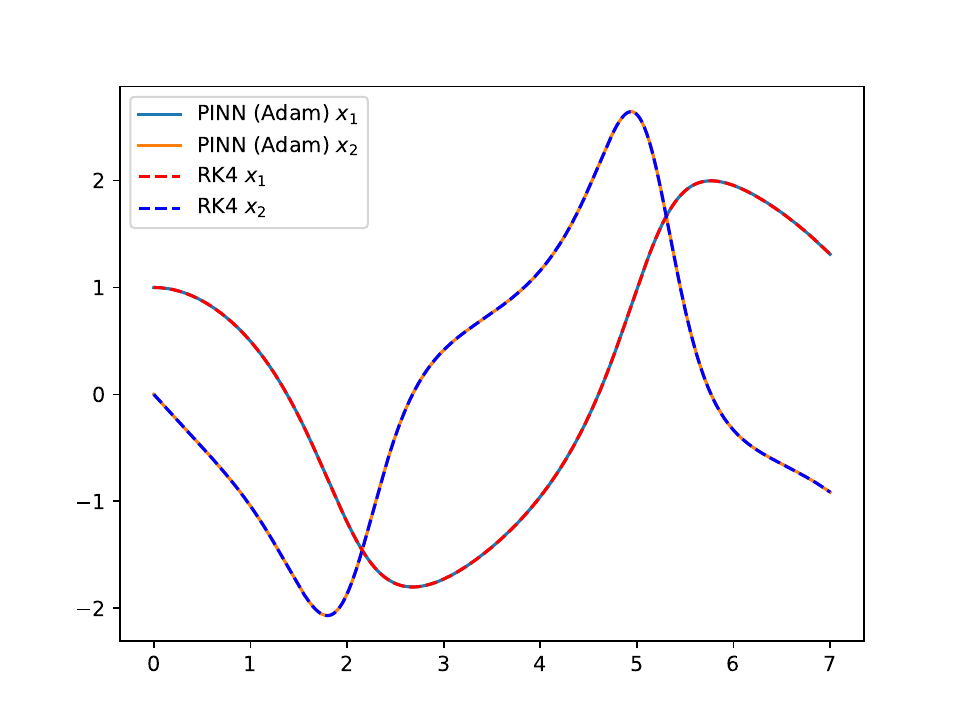}
    \includegraphics[width=0.47\textwidth]{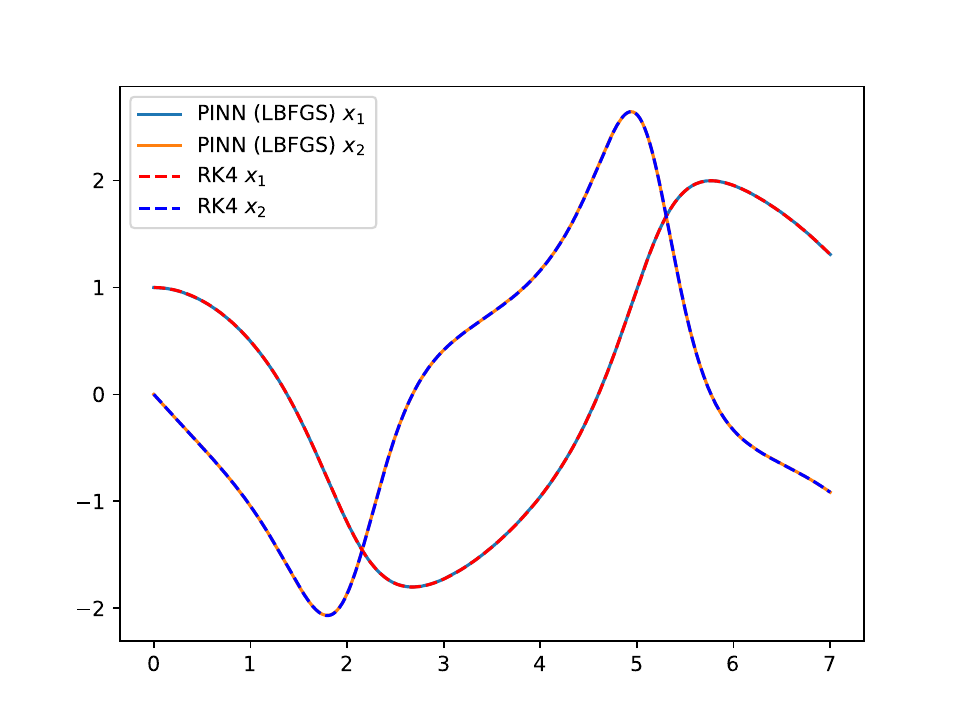}
    \caption{Neural network solutions with three hidden layers and 10 neurons each, compared with the fourth-order Runge-Kutta method.}
    \label{fig:van_der_pol}
\end{figure}

\subsection{PDE Examples}

\begin{figure}[ht]
    \centering
    \includegraphics[width=0.97\linewidth]{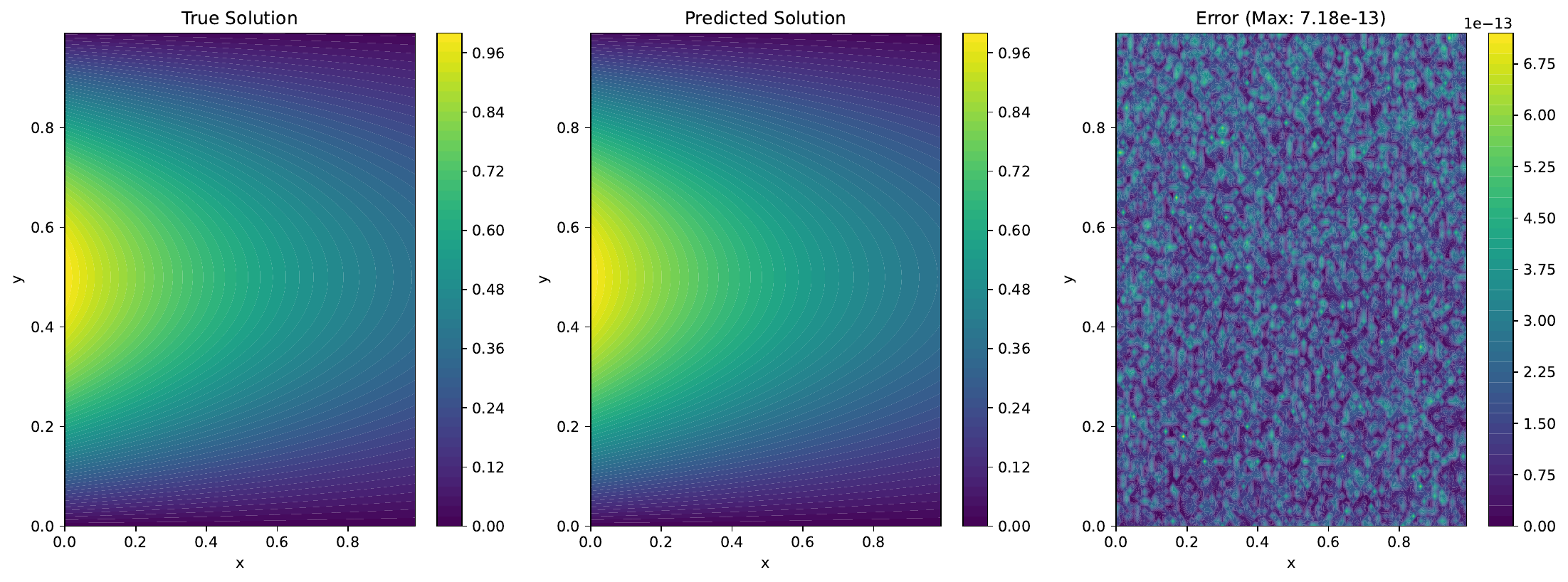}
    \includegraphics[width=0.97\linewidth]{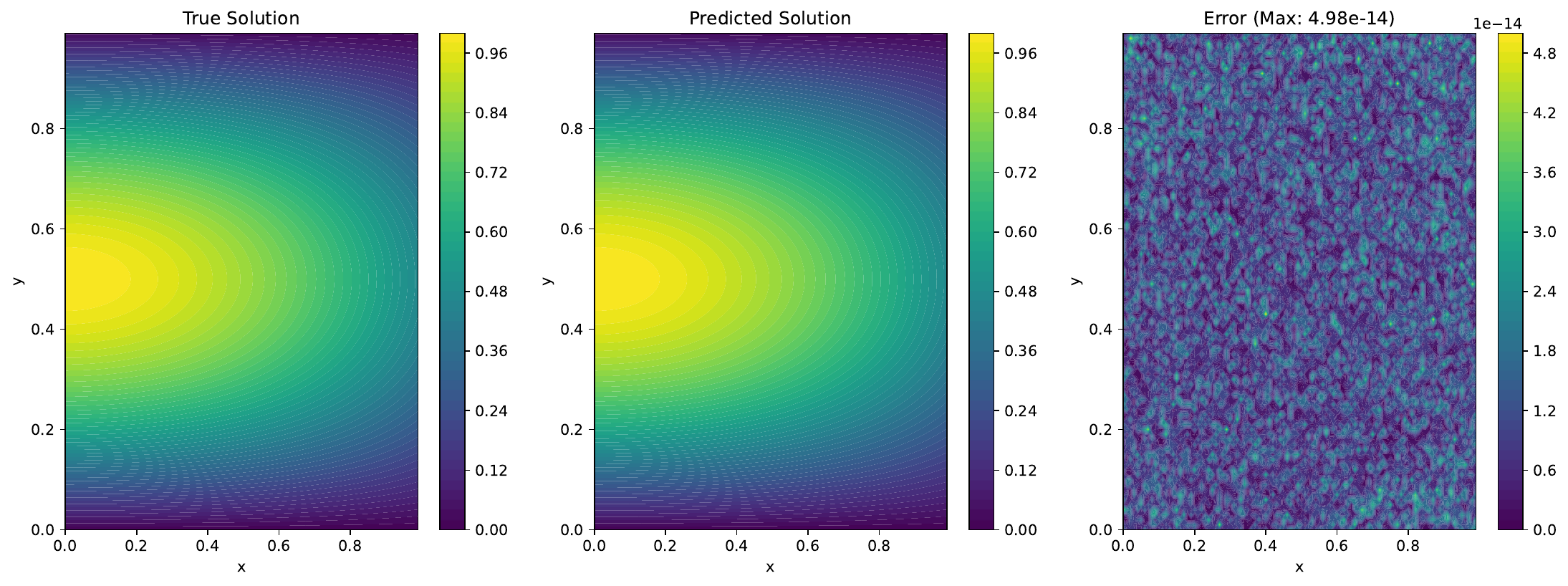}
    \caption{ELM solutions width $m=1600$ of the 1D heat equation (top) and 1D wave equation (bottom), compared with analytical solution. The error is computed empirically via finite samples in the domain.}
    \label{fig:heat_wave1D}
\end{figure}

For elliptic, parabolic, and hyperbolic examples, we employ extreme learning machines (ELMs) with a single hidden layer of width $m=1600$. Hidden-layer weights and biases are randomly initialized and fixed, while output weights are trained via linear least squares. All experiments were conducted on a single NVIDIA H100 GPU.

For each PDE, we set the domain to $\Omega_T=[0,1]^2$. We sample 1600 points randomly from the uniform distribution $U[0,1]^2$ to enforce the PDE residual loss. For each boundary or initial condition, we sample 100 additional points from $U[0,1]$ and fix $t=0$ (inital) or $x=0$, $x=1$ (boundary) to enforce the respective constraints.

Figure \ref{fig:heat_wave1D} shows the neural solutions compared with the analytical solutions along with a plot of their estimated error based on finite samples.

\textbf{Allen-Cahn Equation.}
For the nonlinear Allen-Cahn example, we use a PINN architecture with 8 hidden layers and a width of 20.
As shown in \cref{prop:allen-cahn}, error bounds can be computed independent of the true solution if we re-parameterize the input space to force periodicity. We obtain this via the sinusoidal feature map
\[
x \to (\sin(\pi x), \cos(\pi x)).
\]

\subsection{Verification Procedure}\label{sec:verification}

All certified bounds in linear PDE settings are computed using autoLiRPA \cite{xu2020automatic}.

\textbf{Domain Partitioning.}
Directly estimating residual or solution bounds over the entire computational domain leads to significant over-estimation due to relaxation conservativeness. To mitigate this effect, we partition the domain into $100^n$ uniform subgrids.

Here, $n$ denotes the intrinsic dimension of the quantity being verified:
\begin{itemize}
    \item $n=2$ when certifying PDE residuals or reference errors defined over space-time domains,
    \item $n=1$ when certifying boundary or initial
    conditions defined over lower-dimensional subsets.
\end{itemize}

Verification is then performed independently on each subdomain.

\textbf{$L_\infty$ Bounds.}
For uniform norm certification, we compute a certified $L_\infty$ upper bound of the target quantity on each subdomain using backward linear relaxation. The global $L_\infty$ bound is then obtained by taking the maximum over all subdomains.

\textbf{$L_2$ Bounds.}
For $L_2$-based estimates (e.g., Proposition~\ref{prop:parabolic2}), we compute certified $L_\infty$ bounds on each subdomain, square these bounds, and aggregate them using a Riemann-type sum over the partition. This yields a certified upper bound on the $L_2$ norm over the entire domain.





\end{document}